\documentclass{article}

\PassOptionsToPackage{numbers, compress}{natbib}

\usepackage[preprint]{neurips_2025}

 \usepackage{amsmath}
\usepackage[utf8]{inputenc} 
\usepackage[T1]{fontenc}    
\usepackage{hyperref}       
\usepackage{url}            
\usepackage{booktabs}       
\usepackage{amsfonts}       
\usepackage{nicefrac}       
\usepackage{microtype}      
\usepackage{xcolor}         
\usepackage{xpatch}
\usepackage{amsthm} 
\newtheorem{theorem}{Theorem} 
\newtheorem{lemma}{Lemma}   
\newtheorem{proposition}{Proposition} 
\usepackage{xpatch}
\usepackage{amsthm}
\usepackage{mdframed} 
\usepackage{graphicx}
\usepackage{algorithm}
\usepackage{algpseudocode}
\usepackage{mathtools}
\usepackage{wrapfig}

\usepackage{tikz} 
\usepackage[utf8]{inputenc} 
\usepackage[T1]{fontenc}    
\usepackage{url}            
\usepackage{microtype}      
\usepackage{algorithm}
\usepackage{algpseudocode}
\usepackage{listings}
\usepackage{overpic}
\usepackage{wrapfig}
\usepackage{colortbl}
\usepackage{tabularx}
\usepackage{verbatim}
\usepackage{nicefrac}
\usepackage{empheq}         

\usepackage{booktabs}   
\usepackage{caption} 
\usepackage{subcaption}
\usepackage{multirow}       
\usepackage{makecell}       
\usepackage{adjustbox}      
\usepackage{graphicx}    
\usepackage{array}

\usepackage{tikz}
\usepackage{amsmath} 

\usepackage{centernot}
\usepackage{xcolor}
\definecolor{citecolor}{HTML}{0071BC}
\definecolor{linkcolor}{HTML}{ED1C24}

\usetikzlibrary{arrows.meta, positioning, calc}

\definecolor{myorange}{RGB}{255,127,0} 
\definecolor{myblue}{RGB}{0,114,189}  

\newcommand{\gc}[1]{\textcolor{gray}{#1}}
\usepackage{amsthm}
\usepackage{amssymb}
\usepackage{amsmath}
\usepackage{amsfonts}       
\usepackage{nicefrac}       
\usepackage{mathtools}
\newmdtheoremenv[
  outerlinewidth=1pt,
  roundcorner=5pt,
  innertopmargin=10pt,
  innerbottommargin=10pt,
  innerleftmargin=10pt,
  innerrightmargin=10pt,
  backgroundcolor=gray!5,
  linecolor=black
]{boxedtheorem}{Theorem}

\newmdenv[
  linecolor=black,
  outerlinewidth=1pt,
  roundcorner=5pt,
  innertopmargin=8pt,
  innerbottommargin=8pt,
  backgroundcolor=blue!5, 
  userdefinedwidth=\textwidth
]{highlightbox}

\makeatletter
\xapptocmd{\NAT@bibsetnum}{\setlength{\leftmargin}{0pt}\setlength{\itemindent}{\labelwidth}\addtolength{\itemindent}{\labelsep}}{}{}
\makeatother

\title{P-Guide: Parameter-Efficient Prior Steering for Single-Pass CFG Inference}

\author{%
  Xin Peng$^{1}$ \quad
  Ang Gao$^{1}$\thanks{Corresponding author} \\
  $^{1}$School of Physical Science and Technology, \\
  Beijing University of Posts and Telecommunications, Beijing, China \\
  \texttt{anggao@bupt.edu.cn}
}

\begin{document}

\maketitle

\begin{abstract}
Classifier-Free Guidance (CFG) is essential for high-fidelity conditional generation in flow matching, yet it imposes significant computational overhead by requiring dual forward passes at each sampling step. In this work, we address this bottleneck by introducing \textbf{P-Guide}, a framework that achieves high-quality guidance through a single inference pass by modulating only the initial latent state. We further show that, under a first-order approximation, P-Guide is equivalent to CFG in the sense that it steers generation from the prior space, without requiring explicit velocity field extrapolation during sampling. We consider both homoscedastic and \textbf{heteroscedastic} priors, and find that jointly modeling the mean and variance enables adaptive loss attenuation and improved robustness to data uncertainty. Extensive experiments demonstrate that P-Guide reduces inference latency by approximately 50\% while maintaining fidelity and prompt alignment competitive with standard dual-pass CFG baselines.
\end{abstract}

\section{Introduction}

Continuous-time generative modeling frameworks, such as Flow Matching (FM) \cite{lipman2022flow,albergo2023building} and Rectified Flow \cite{liu2023flow}, have established a powerful paradigm for transforming simple prior distributions into complex data manifolds via deterministic Ordinary Differential Equations (ODEs) \cite{song2021scorebased, song2020score}. This line of methods is closely connected to score-based diffusion models \cite{ho2020ddpm, song2021ddim, nichol2021improved}, and has recently progressed toward scalable generative architectures, combining flow-based modeling paradigms with transformer-based architectures \cite{peebles2023dit, esser2024scaling}.

To achieve precise semantic control in large-scale text-to-image synthesis, Classifier-Free Guidance (CFG) \cite{ho2022classifier} has become an indispensable technique in modern generative systems, including Stable Diffusion \cite{rombach2022latent}, SDXL \cite{podell2023sdxl}, Imagen \cite{saharia2022imagen}, DALL·E 2 \cite{ramesh2022dalle2}, DALL·E \cite{ramesh2021dalle}, GLIDE \cite{nichol2021glide}, and recent industrial-scale models such as Stable Diffusion 3 / 3.5 \cite{esser2024scaling, sd35_technical}, FLUX \cite{flux2024}, and large multimodal systems like Qwen-Image \cite{qwenimage2024} and Wanxiang \cite{wanxiang2024}. By linearly extrapolating between conditional and unconditional velocity fields, CFG significantly sharpens prompt alignment and enhances structural fidelity across diffusion and flow-based architectures \cite{dhariwal2021diffusion, esser2024scaling}.

\begin{figure}[t]
  \centering
  \includegraphics[width=\linewidth]{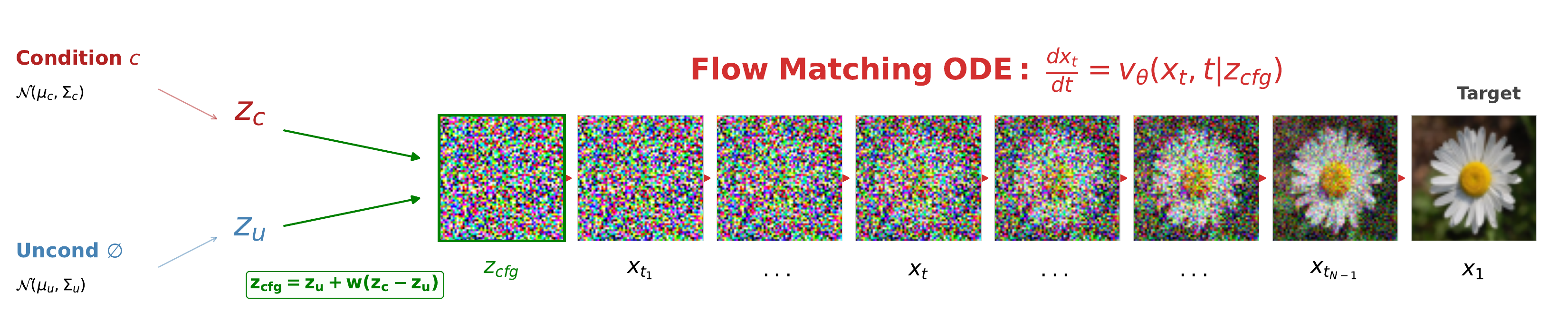}
  \caption{\textbf{Conceptual overview of the P-Guide framework.} Standard CFG requires evaluating the velocity field twice (\(v_{cond}, v_{uncond}\)) at each integration step. P-Guide relocates this guidance to the origin by modulating the initial noise state using a data-dependent prior. Based on our trajectory-level approximation, this initial shift anchors the global ODE path, enabling high-fidelity conditional generation with only a \textbf{single inference pass}.}
  \label{fig:p_guide_schematic}
\end{figure}

Despite its success, the standard implementation of CFG faces a major computational bottleneck: the dual-pass requirement. At each sampling step, the model must be evaluated twice—once for the conditional branch and once for the unconditional one. This doubles the inference latency and remains the primary obstacle for deploying high-fidelity generative flows in real-time or resource-constrained environments, especially in large-scale systems such as SDXL \cite{podell2023sdxl}, SD3 \cite{esser2024scaling}, and FLUX \cite{flux2024}. While existing methods attempt to mitigate this via distillation or step-truncation \cite{salimans2022progressive, meng2022distillation, song2023consistency, luhman2023latent}, they often require intensive retraining or compromise the geometric consistency of the generation path.

In this work, we propose to mitigate this bottleneck by shifting the guidance mechanism from the sampling process to the trajectory origin (see Figure~\ref{fig:p_guide_schematic}). We introduce \textbf{P-Guide} (Prior-Guide), a parameter-efficient framework that achieves conditional control by modulating the initial state in the prior space. Our approach is motivated by a \textbf{trajectory-level approximation perspective}, where we show that a structured shift in the prior space is consistent with velocity field extrapolation along the ODE trajectory under a first-order approximation. By steering the generation from its starting point, P-Guide enables high-quality guided sampling with only a \textbf{single inference pass} per step, leading to a substantial reduction in the computational cost of CFG.

Furthermore, we explore both homoscedastic and heteroscedastic formulations of the prior, where the latter jointly models the conditional mean and variance in the latent space. This allows the model to capture data-dependent uncertainty and implement an adaptive loss attenuation mechanism, improving robustness against noisy conditions. Through extensive experiments on benchmark datasets including ImageNet-1k, we demonstrate that P-Guide delivers a $2\times$ speedup in CFG inference while maintaining fidelity and prompt alignment competitive with standard dual-pass baselines, requiring only a small number of additional parameters.

\paragraph{Our Contributions.}
We shift the guidance mechanism from the sampling process to the trajectory's origin. Our technical contributions include:
\begin{itemize}
    \item \textbf{Trajectory-Level approximation :} A first-order analysis showing that a structured shift in the prior space is consistent with velocity field extrapolation along the ODE evolution.
    \item \textbf{P-Guide Framework:} A parameter-efficient approach enabling single-pass CFG inference by modulating the initial state, reducing the computational cost of guidance by 50\%.
    \item \textbf{Heteroscedastic Modeling:} We study both homoscedastic and heteroscedastic priors, and show that joint mean-variance modeling captures data-dependent uncertainty and enables adaptive loss attenuation.
\end{itemize}

\section{Related Work}

\paragraph{Diffusion and Flow-Based Generative Models.}
Generative modeling has transitioned from stochastic diffusion models \cite{ho2020ddpm, song2021ddim, nichol2021improved, song2021scorebased} to continuous-time frameworks such as Flow Matching (FM) and Rectified Flow \cite{lipman2022flow, liu2023flow}. While diffusion models rely on iterative score estimation to reverse a noising process, FM instead regresses a time-dependent velocity field to transport a simple prior toward the data manifold \cite{lipman2022flow}. Rectified Flow further improves efficiency by learning near-straight trajectories, thereby reducing discretization errors in Ordinary Differential Equation (ODE) solvers \cite{liu2023flow}. These approaches have recently been extended to large-scale transformer-based generative models \cite{peebles2023dit, esser2024scaling}.

\paragraph{Informed and Data-Dependent Priors.}
Beyond standard uninformed priors, recent works demonstrate that context-aware distributions can enhance conditioning efficiency and reduce transport complexity \cite{lee2022priorgrad, crossflow2023, carflow2023, warmstart2023}. Specifically, PriorGrad \cite{lee2022priorgrad} and Warm-Start Diffusion \cite{warmstart2023} adapt priors to conditional statistics to simplify the generative task, while CAR-Flow \cite{carflow2023} and CrossFlow \cite{crossflow2023} relocate distributions via learned shifts or cross-modal mappings to align source and target manifolds. Despite improving training convergence or modeling flexibility, these techniques do not address the inherent dual-pass computational bottleneck imposed by classifier-free guidance during inference.

\paragraph{Conditional Generative Modeling and Classifier-Free Guidance.}
Classifier-Free Guidance (CFG) \cite{ho2022classifier} has become the de facto standard for controllable generation, enhancing semantic alignment by extrapolating between conditional and unconditional predictions. It is widely adopted in modern text-to-image systems, including Stable Diffusion \cite{rombach2022latent}, SDXL \cite{podell2023sdxl}, and more recent flow-based large-scale models such as Stable Diffusion 3 \cite{esser2024scaling} and FLUX \cite{flux2024}. However, the standard implementation requires two forward passes per sampling step, doubling inference cost and limiting deployment in real-time or resource-constrained settings.

\paragraph{Analysis and Control of Guidance Dynamics.}
Recent work has investigated the geometric properties and stability of guided sampling trajectories. CFG-Zero* \cite{cfgzero2023} improves early-stage guidance behavior via optimized scaling, while CFG-MP \cite{cfgmp2023} and Rectified-CFG++ \cite{rectifiedcfgpp2024} incorporate manifold constraints or predictor–corrector schemes to maintain trajectory consistency. CFG-Ctrl \cite{cfgctrl2024} further formulates guidance as a control problem, providing a dynamical systems perspective. More recently, C$^2$FG \cite{gao2026c2fgcontrolclassifierfreeguidance} analyzes guidance through score discrepancy and introduces a principled control mechanism to stabilize conditional generation. Nevertheless, these methods primarily modify the guidance dynamics within the sampling trajectory and do not eliminate the inherent dual-pass computational bottleneck of CFG.

\section{Background}

\subsection{Flow Matching}
Flow Matching (FM) provides a simulation-free framework to train continuous normalizing flows by regressing a time-dependent velocity field $v_{\theta}(x_t, t)$ \cite{lipman2022flow}. Given a source distribution $p_0$ (typically standard Gaussian) and a target data distribution $p_1$, FM constructs a probability path $p_t$ that transports $p_0$ to $p_1$, and is closely related to diffusion models and score-based generative modeling \cite{ho2020ddpm, song2021scorebased, song2021ddim}.

The sample trajectory $x_t$ is governed by the following ODE:
\begin{equation}
    \frac{dx_t}{dt} = v_{\theta}(x_t, t), \quad x_0 \sim p_0.
\end{equation}

For a pair of samples $(x_0, x_1)$, a commonly used construction is the independent conditional probability path, which uses linear interpolation: $x_t = (1 - t)x_0 + tx_1$ \cite{lipman2022flow}. The corresponding target velocity is constant: $u_t(x_t | x_0, x_1) = x_1 - x_0$. The model is trained using the Flow Matching objective:
\begin{equation}
    \mathcal{L}_{FM}(\theta) = \mathbb{E}_{t, x_0, x_1} [\| v_{\theta}(x_t, t) - (x_1 - x_0) \|^2].
\end{equation}

Rectified Flow is a special case of Flow Matching that encourages straight-line trajectories to reduce numerical integration error and improve sampling efficiency \cite{liu2023flow}. Recent extensions further connect FM with transformer-based diffusion models and large-scale generative systems \cite{peebles2023dit, esser2024scaling}.

\subsection{Classifier-Free Guidance (CFG)}
Classifier-Free Guidance (CFG) is the dominant technique for controllable generative modeling in diffusion and flow-based systems \cite{ho2022classifier}. It has become a core component in modern text-to-image systems such as Stable Diffusion \cite{rombach2022latent}, SDXL \cite{podell2023sdxl}, Imagen \cite{saharia2022imagen}, DALL·E 2 \cite{ramesh2022dalle2}, GLIDE \cite{nichol2021glide}, and recent large-scale models including Stable Diffusion 3 \cite{esser2024scaling} and FLUX \cite{flux2024}.

During training, the condition is randomly replaced with a null token $\emptyset$, allowing a single model to learn both conditional velocity $v(x_t, t | y)$ and unconditional velocity $v(x_t, t | \emptyset)$. At inference time, the guided velocity field is computed as:
\begin{equation}
    v_{cfg}(x_t, t) = v_{\theta}(x_t, t | \emptyset) + w (v_{\theta}(x_t, t | y) - v_{\theta}(x_t, t | \emptyset)),
\end{equation}
where $w \geq 1$ is the guidance scale \cite{ho2022classifier}. While CFG significantly improves semantic alignment and sample fidelity, it requires evaluating the neural network twice at every integration step, which doubles computational cost and limits real-time deployment in large-scale systems such as SDXL and FLUX \cite{podell2023sdxl, flux2024}.

\section{Method}

In this section, we present the \textbf{P-Guide} (Prior-Guide) framework. Our core innovation lies in shifting the guidance mechanism from computationally expensive velocity field extrapolation to a structured modulation of the initial state in the prior space.

\subsection{Prior Space Parameterization}
\label{subsec:hetero_prior}
Standard generative flows typically transport an uninformed Gaussian prior $p_0(z) = \mathcal{N}(0, \mathbf{I})$ to a target data distribution $p_1$ \cite{ho2020ddpm, lipman2022flow}. However, recent studies have shown that Gaussian priors are not strictly necessary, and learned or data-dependent priors can significantly improve conditioning efficiency and reduce transport complexity \cite{lee2022priorgrad, crossflow2023, carflow2023, warmstart2023}.
\definecolor{codeblue}{rgb}{0.25,0.5,0.5}
\definecolor{codekw}{rgb}{0.85, 0.18, 0.50}

\definecolor{codesign}{RGB}{0, 0, 255}
\definecolor{codefunc}{rgb}{0.85, 0.18, 0.50}

\lstdefinelanguage{PythonFuncColor}{
  language=Python,
  keywordstyle=\color{blue}\bfseries,
  commentstyle=\color{codeblue},  
  stringstyle=\color{orange},
  showstringspaces=false,
  basicstyle=\ttfamily\small,
  literate=
    {*}{{\color{codesign}* }}{1}
    {-}{{\color{codesign}- }}{1}
    {+}{{\color{codesign}+ }}{1}
    {dataloader}{{\color{codefunc}dataloader}}{1}
    {sample_t_r}{{\color{codefunc}sample\_t\_r}}{1}
    {randn}{{\color{codefunc}randn}}{1}
    {randn_like}{{\color{codefunc}randn\_like}}{1}
    {jvp}{{\color{codefunc}jvp}}{1}
    {stopgrad}{{\color{codefunc}stopgrad}}{1}
    {metric}{{\color{codefunc}metric}}{1}
}

\lstset{
  language=PythonFuncColor,
  backgroundcolor=\color{white},
  basicstyle=\fontsize{9pt}{9.9pt}\ttfamily\selectfont,
  columns=fullflexible,
  breaklines=true,
  captionpos=b,
}

\begin{wrapfigure}{r}{0.5\linewidth}
\vspace{-1.2em}
\centering
\begin{minipage}{0.95\linewidth}

\begin{algorithm}[H]
\caption{{P-Guide}: Training}
\label{alg:pguide_train}
\footnotesize
\setlength{\baselineskip}{8.8pt}

\begin{lstlisting}
# Stage 1: train F_phi

x1, y = sample_data()
mu, sigma = F_phi(y)
loss_prior = (norm(x1 - mu)**2) / (2*sigma**2) 
             + 0.5 * log(sigma**2)


# Stage 2: train v_theta

x1, y = sample_data()
eps = randn_like(x1)
mu, sigma = F_phi(y)
z = mu + sigma * eps

t = sample_uniform()
x_t = (1 - t) * z + t * x1

loss = metric(v_theta(x_t,t) - (x1 - z))
\end{lstlisting}

\end{algorithm}

\vspace{-1.4em}

\begin{algorithm}[H]
\caption{{P-Guide}: Sampling}
\label{alg:pguide_sample}
\footnotesize
\setlength{\baselineskip}{8.8pt}

\begin{lstlisting}
# classifier-free guidance in latent

eps = randn(x_shape)
mu_c, sigma_c = F_phi(y)
mu_u, sigma_u = F_phi(None)
z = mu_u + w*(mu_c - mu_u)
z = z + (sigma_u + w*(sigma_c - sigma_u)) * eps
x = z

for t in schedule:
    x = x + v_theta(x,t) * dt
\end{lstlisting}

\end{algorithm}

\end{minipage}
\vspace{-1.2em}
\end{wrapfigure}

To enable early-stage steering, we parameterize the prior distribution as $p_0(z|y)$ such that the starting point of the trajectory already encodes conditional information.

\paragraph{Homoscedastic P-Guide (Base).}
We first define a baseline modeling scheme where the condition $y$ modulates only the mean of the prior distribution while keeping a fixed unit variance:
\begin{equation}
    z = \epsilon + \mu_{\phi}(y), \quad \epsilon \sim \mathcal{N}(0, \mathbf{I}).
\end{equation}
Under standard regression optimality (see Appendix~\ref{sec:appendix_equivalence} for details), the learnable mapping $\mu_{\phi}(y)$ approximates the conditional expectation $\mathbb{E}[x \mid y]$. We define $z_c = \epsilon + \mu_{\phi}(y)$ and $z_u = \epsilon + \mu_{\phi}(\emptyset)$, so that the initial state shift $z_c - z_u = \mu_{\phi}(y) - \mu_{\phi}(\emptyset)$ serves as a directional anchor for the global trajectory.

\paragraph{Heteroscedastic P-Guide (Extension).}
We extend the above formulation to a heteroscedastic prior, where the condition additionally controls the scale of the latent initialization. Specifically, we parameterize the conditional distribution using neural networks $(\mu_{\phi}(y), \sigma_{\phi}(y))$ and define:
\begin{equation}
    z = \mu_{\phi}(y) + \sigma_{\phi}(y) \odot \epsilon, \quad \epsilon \sim \mathcal{N}(0, \mathbf{I}).
\end{equation}
Here, the initial state $z$ is no longer mere white noise but a "semantic seed" encoding the target mode. We define the conditional seed as $z_c = \mu_{\phi}(y) + \sigma_{\phi}(y)\odot\epsilon$ and the unconditional seed as $z_u = \mu_{\phi}(\emptyset) + \sigma_{\phi}(\emptyset)\odot\epsilon$, where $\emptyset$ represents the null condition.
\subsection{Sequential Two-Stage Training Paradigm}

We propose a decoupled training strategy inspired by probabilistic modeling and heteroscedastic regression \cite{kendall2017uncertainties}, which ensures the prior steering module captures accurate data statistics before the backbone flow model learns transport dynamics \cite{lipman2022flow}.

\paragraph{Stage 1: Heteroscedastic Prior Learning.} 
In the first stage, we train the prior module to minimize the conditional negative log-likelihood (NLL) of the target data $x_1$ given condition $y$. This formulation follows probabilistic maximum likelihood estimation under Gaussian observation noise and enables joint learning of mean and uncertainty \cite{kendall2017uncertainties}:
\begin{equation}
    \mathcal{L}_{prior} = \mathbb{E}_{x_1, y} \left[ \frac{\|x_1 - \mu_{\phi}(y)\|^2}{2\sigma_{\phi}^2(y)} + \frac{1}{2} \log \sigma_{\phi}^2(y) \right].
\end{equation}
This objective induces an adaptive weighting effect, where the gradient with respect to the mean estimator is scaled by the inverse of the predicted variance, thereby attenuating the influence of noisy samples. As a result, regions with high aleatoric uncertainty are down-weighted during training, leading to a more stable semantic prior representation \cite{kendall2017uncertainties}.

\paragraph{Stage 2: Flow Matching.}
After convergence, we freeze the parameters and train the velocity field $v_{\theta}$ using a modified Flow Matching objective, where starting points $z$ are sampled from the learned conditional prior $p_0(z|y)$ defined by the frozen module. We define the conditional probability path as $x_t = (1-t)z + t x_1$, yielding the objective:
\begin{equation}
    \mathcal{L}_{FM}(\theta) = \mathbb{E}_{t, x_1, z \sim p_0(z|y)} \left[ \| v_{\theta}(x_t, t) - (x_1 - z) \|^2 \right].
\end{equation}
Since $z$ is already aligned with condition $y$ and geographically closer to the target manifold, the velocity field $v_{\theta}$ learns straighter trajectories, which is consistent with the efficiency principle of neural ODE-based generative modeling \cite{chen2018neuralode}. This significantly simplifies the transport task and accelerates convergence.

\subsection{trajectory-level approximation}

Traditional CFG operates by linearly extrapolating between conditional and unconditional velocity fields $v_t$ at every integration step \cite{ho2022classifier}. We provide an approximation perspective for P-Guide by showing that this mechanism can be interpreted as a linear shift in the prior space under the assumption of flow consistency \cite{chen2018neuralode, lipman2022flow}.

\begin{boxedtheorem}
\label{thm:equivalence}

Let $\Phi_t: \mathcal{Z} \to \mathcal{X}$ be a flow map induced by a velocity field $v(x_t, t \mid z) = \frac{d}{dt}\Phi_t(z)$, where $z$ denotes the initial latent state and $x_t = \Phi_t(z)$ is the trajectory at time $t$.  
We denote the trajectories starting from the conditional and unconditional initial states as $v(x_t, t \mid z_c)$ and $v(x_t, t \mid z_u)$, respectively.

Under a trajectory linear response assumption, where the flow responds approximately linearly to perturbations in the initial latent space, the velocity difference satisfies:
\begin{equation}
    v(x_t, t \mid z_c) - v(x_t, t \mid z_u) \propto z_c - z_u .
\end{equation}

\end{boxedtheorem}

The result follows from a linear response approximation of the flow dynamics, where changes in the initial latent state propagate approximately linearly along the trajectory. Here, the conditional velocity $v(x_t, t \mid z_c)$ and unconditional velocity $v(x_t, t \mid z_u)$ refer to the velocity fields induced by trajectories initialized from the conditional and unconditional latent states $z_c$ and $z_u$, respectively. In this regime, their difference is governed by the difference between the corresponding initial states. A more detailed derivation based on flow linearization and connections to score-based guidance is provided in Appendix~\ref{sec:appendix_equivalence}. This perspective naturally leads to the interpretation that CFG velocity can be directly controlled by the initial latent shift.

\subsection{Inference: Single-Pass CFG via Prior Steering}

Leveraging the first-order equivalence established in Theorem \ref{thm:equivalence}, \textbf{P-Guide} performs guidance by applying CFG in the prior space before ODE integration begins, i.e., conditioning is injected into the initial latent distribution. As a result, the ODE is solved only once using a single velocity field evaluation, replacing the standard CFG procedure that requires two velocity evaluations per step during integration.

\begin{highlightbox}
\textbf{The P-Guide Core Formula:} We define the guided initial state $z_{cfg}$ by performing the linear combination of learned distribution parameters directly in the prior space:
\begin{equation}
    \label{eq:p_guide_core}
    z_{cfg} = \mu_{\phi}(\emptyset) + w (\mu_{\phi}(y) - \mu_{\phi}(\emptyset)) + [\sigma_{\phi}(\emptyset) + w (\sigma_{\phi}(y) - \sigma_{\phi}(\emptyset))] \odot \epsilon.
\end{equation}
\end{highlightbox}
Once $z_{cfg}$ is computed, generation proceeds by solving the probability flow ODE $\frac{dx_t}{dt} = v_{\theta}(x_t, t)$ starting from $x_0 = z_{cfg}$. This \textbf{single inference pass} completely avoids the need for dual model evaluations per step, reducing the computational cost of CFG by $50\%$.

\subsection{Toy Experiment: Visualizing Trajectory Equivalence}
\label{subsec:toy_experiment}

To provide empirical intuition for the \textbf{P-Guide} framework, we conduct a 2D toy experiment. This visualization demonstrates how a structured shift at the starting point effectively anchors the global trajectory, replacing the need for iterative velocity corrections.

\begin{figure}[h]
  \centering
  \includegraphics[width=\linewidth]{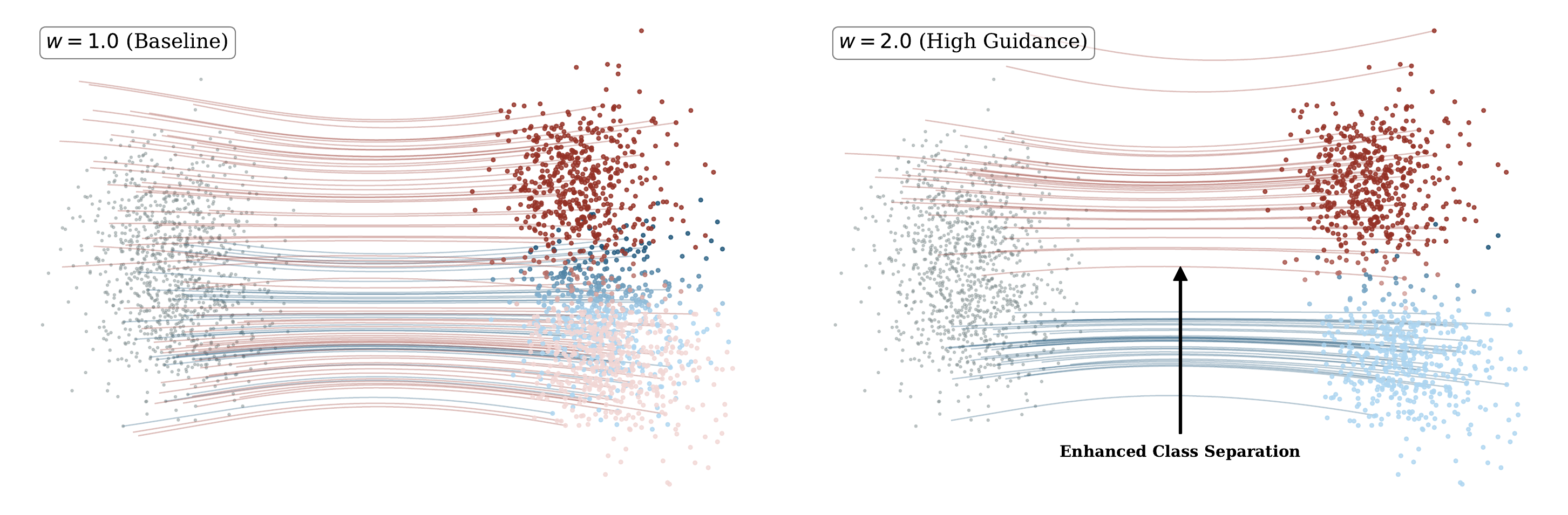}
  \caption{\textbf{Visualizing trajectory steering in a 2D toy setup.} Standard CFG requires per-step corrections in velocity space (dual-pass). P-Guide utilizes the prior shift $z_c - z_u$ as a directional anchor at $t=0$. As shown, increasing $w$ from $1.0$ to $2.0$ causes trajectories to concentrate and separate precisely toward target modes using only a \textbf{single inference pass}.}
  \label{fig:toy_viz}
\end{figure}

\paragraph{Experimental Setup.} 
We construct a 2D dataset where the target distribution $\pi_1$ is multimodal, consisting of two distinct clusters at a radius of $5.0$ units with angles of $0$ and $4$ radians ($\sigma=0.5$).  We train a Rectified Flow model \cite{liu2023flow} to learn the velocity field $v_{\theta}(x_t, t | y)$. Simultaneously, we pre-train the prior module $F_{\phi}(y)$ to estimate the conditional mean $\mu$ of the data clusters.

\paragraph{Trajectory Anchoring through Prior Steering.} 
Standard CFG necessitates two model evaluations per step to compute an extrapolated velocity. In contrast, P-Guide injects the guidance at $t=0$ by shifting the initial state $z$ toward the target mode according to $z_{cfg} = z_u + w(z_c - z_u)$. 

Visualizing the resulting trajectories in Figure \ref{fig:toy_viz} confirms that this initial shift acts as a directional anchor. Since the flow model learns to approximate straight-line transport between distributions, the semantic bias introduced at the origin naturally steers the entire deterministic ODE path. As the guidance scale $w$ increases, the trajectories for different classes exhibit enhanced separation and concentrate more sharply on their respective target modes. This demonstrates that high-quality conditional generation is achievable through a \textbf{single inference pass}, as the trajectory is correctly anchored before sampling begins.

\section{Experiments}

In this section, we evaluate the empirical performance and computational efficiency of \textbf{P-Guide} across three benchmark datasets: MNIST \cite{lecun1998mnist}, CIFAR-10 \cite{krizhevsky2009cifar}, and ImageNet-1k \cite{imagenet2009}. Our objective is to demonstrate that by shifting guidance to the prior space, we achieve a near-doubling of inference speed with negligible parameter overhead and no sacrifice in generation quality.

\subsection{Experimental Setup}
\label{subsec:experimental_setup}

\paragraph{Hardware and Implementation.} All experiments were conducted on a single workstation equipped with one \textbf{NVIDIA GeForce RTX 4090 GPU} (24GB VRAM). We adopt the Rectified Flow (RF) objective \cite{liu2023flow} with a linear probability path $x_t = (1-t)z + t x_1$. For ImageNet-1k, the generative process operates in a latent space compressed by a pre-trained \textbf{AutoKL} model \cite{esser2024scaling} with a storage footprint of \textbf{79.779 MB}. Following our two-stage paradigm, the prior steering module $F_{\phi}(y)$ is first pre-trained to convergence via Negative Log-Likelihood (NLL), consistent with probabilistic modeling in heteroscedastic regression \cite{kendall2017uncertainties}. In our main results, $F_{\phi}(y)$ follows the configuration described in Section~\ref{subsec:hetero_prior}, where the variance $\sigma_{\phi}(y)$ is learnable. 

\paragraph{Extreme Parameter Efficiency.}
A key advantage of P-Guide is its extremely small parameter footprint. The proposed PG module introduces only \textbf{1.247 MB} of additional storage. This minimal overhead makes it lightweight relative to the overall generative model and enables straightforward integration into existing architectures without modifying the backbone design.
We further analyze the effect of scaling the capacity of $F_{\phi}(y)$ in Appendix~\ref{appendix:ablation_param_scale}.

\paragraph{Metrics.}
We report \textbf{Fr\'{e}chet Inception Distance (FID)} \cite{heusel2017fid}, 
\textbf{Spatial FID (sFID)} \cite{parmar2022sfid}, and 
\textbf{Inception Score (IS)} \cite{salimans2016is} 
to quantify sample quality. 
All metrics are computed using 50k generated samples following standard evaluation protocols. 
Efficiency is measured via hardware-agnostic \textbf{GFLOPs} per sampling step. 
We also report total \textbf{Model Size (MB)} to evaluate parameter efficiency.

To evaluate the effectiveness of classifier-free guidance (CFG), we further measure \textbf{generation accuracy} on conditionally generated samples. 
Specifically, for MNIST we use a pretrained CNN classifier achieving 99.1\% accuracy \cite{cnn_mnist_ref}, and for CIFAR-10 we use a pretrained ResNet-56 achieving 92.0\% accuracy \cite{resnet56_cifar_ref}. 
These classifiers are fixed and used solely to assess whether generated samples are consistent with the target labels under different guidance strengths.

\subsection{Results on MNIST and CIFAR-10}
\label{subsec:low_res_results}

For low-resolution datasets, we adopt different backbones for each setting: we use a standard \textbf{U-Net}~\cite{ronneberger2015u,ho2020ddpm} for CIFAR-10, and \textbf{LightningDiT}~\cite{yao2025reconstruction} for MNIST. We compare our single-pass P-Guide against the standard dual-pass Vanilla CFG baseline. For a fair comparison, both CFM and P-Guide are trained for 400K steps under identical settings.
\begin{figure*}[h]
  \centering
  \begin{minipage}{0.48\textwidth}
    \centering
    \includegraphics[width=\linewidth]{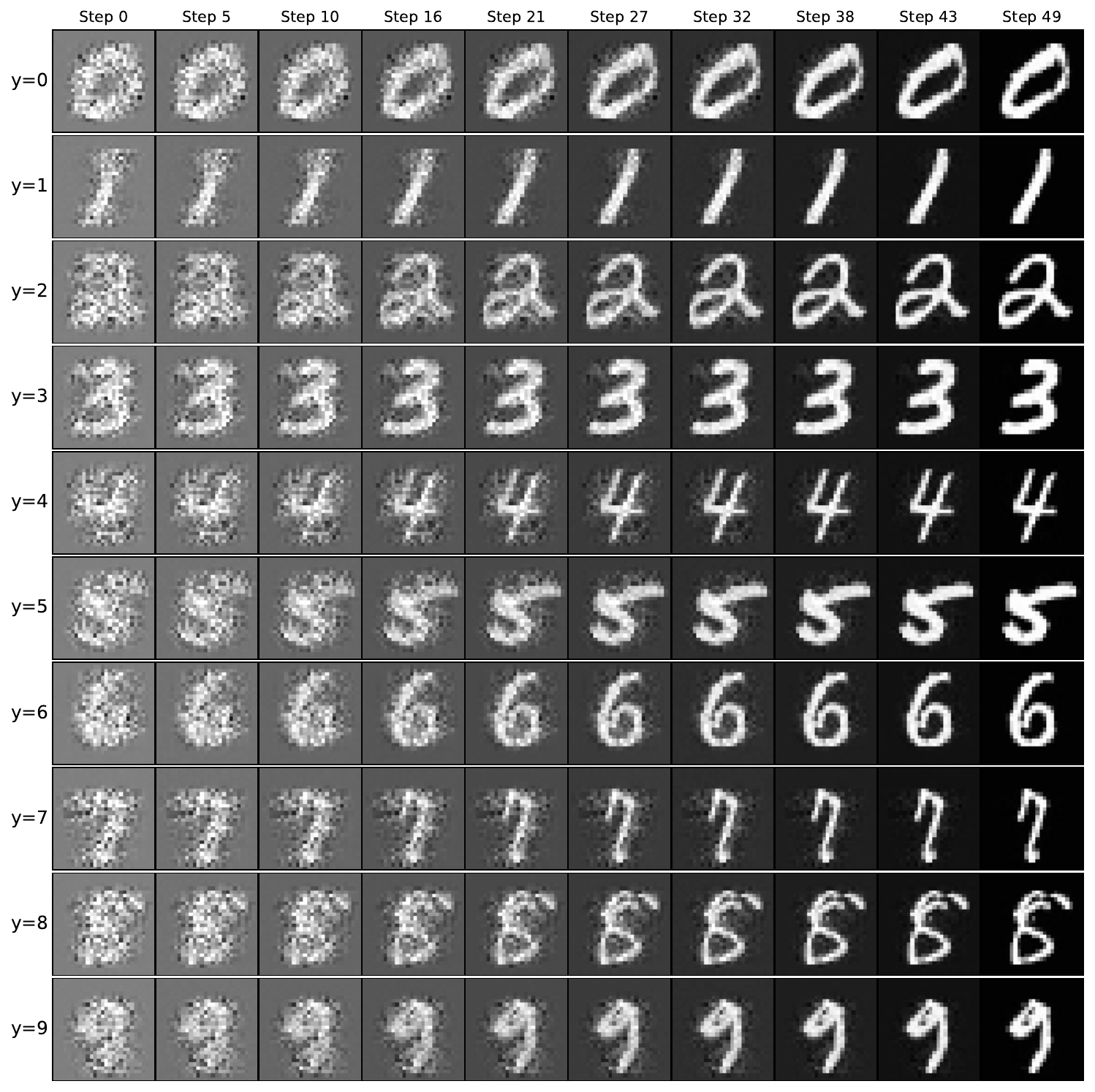}
    \small (a) Trajectory from $z_{cfg}$ to $x_1$ with $w=1.0$
  \end{minipage}
  \hfill
  \begin{minipage}{0.48\textwidth}
    \centering
    \includegraphics[width=\linewidth]{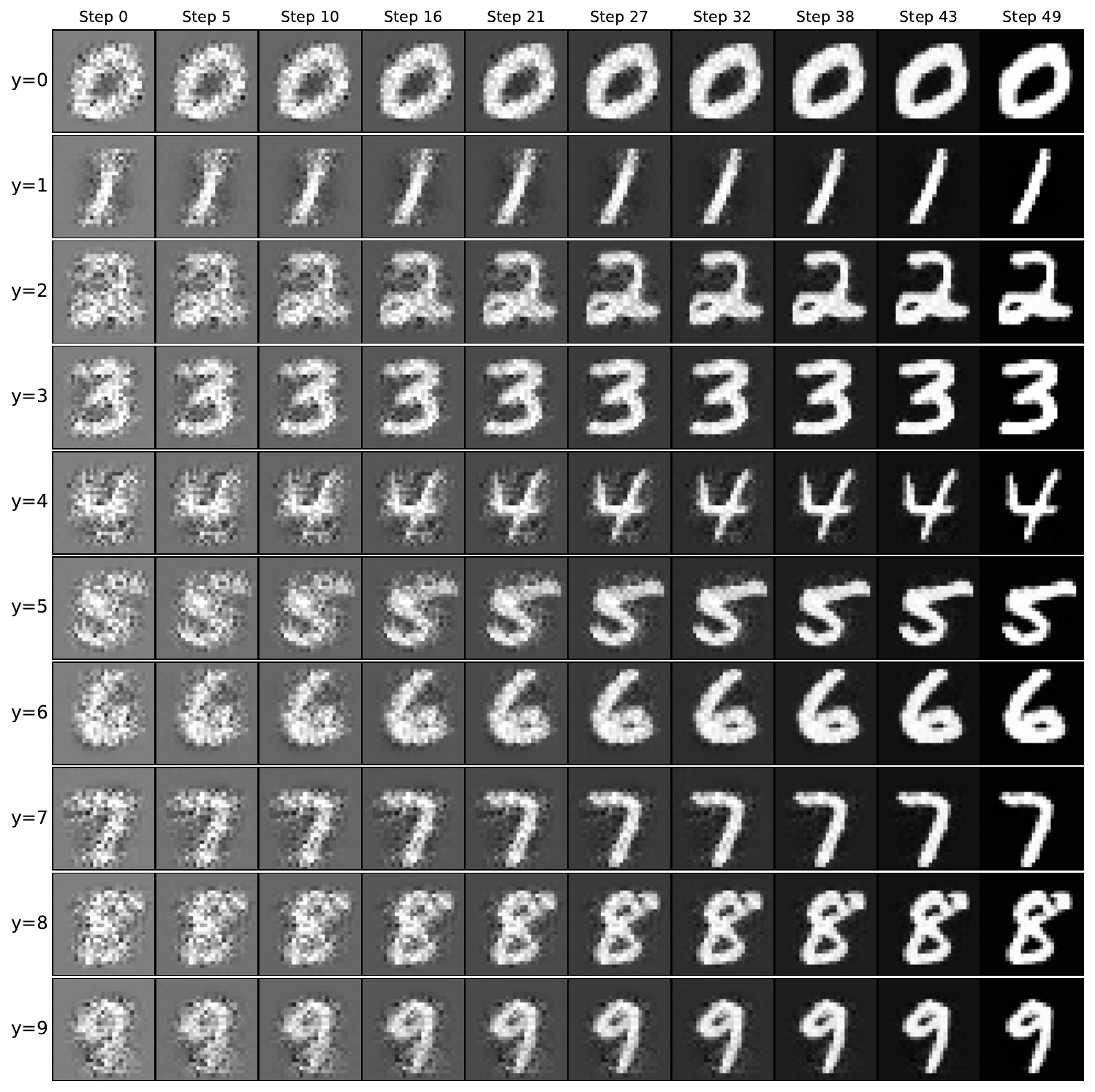}
    \small (b) Trajectory from $z_{cfg}$ to $x_1$ with $w=1.5$
  \end{minipage}
  \caption{\textbf{Visual comparison of P-Guide generation trajectories on MNIST.} Each row shows the evolution from prior ($t=0$) to data ($t=1$). Increasing the guidance scale ($w=1.0 \rightarrow 1.5$) sharpens semantic structure from early steps, confirming effective trajectory-level control from the origin.}
  \label{fig:mnist_trajs}
\end{figure*}
\paragraph{Trajectory Generation Evolution.} 
To intuitively understand the guidance mechanism, we visualize the step-by-step generation process of MNIST digits using the same random seed under two different guidance scales in Figure \ref{fig:mnist_trajs}. 

\paragraph{Quality and Throughput.}
As shown in Table~\ref{tab:low_res_exhaustive}, we sweep the guidance scale $w$ for all methods and observe that increasing $w$ generally improves classification accuracy, indicating stronger conditional alignment. Standard CFM achieves this by combining conditional and unconditional velocity fields at each step, but incurs a $2\times$ computational cost when $w>1.0$ (e.g., GFLOPs increase from $417.6$ to $835.2$ on MNIST), whereas \textbf{P-Guide} performs guidance in prior space and maintains nearly constant computational cost across all $w$ with negligible overhead. Around $w\approx1.0$, the heteroscedastic variant $\text{PG}(\sigma_\phi(y))$ consistently outperforms the fixed-variance version $\text{PG}(\sigma=1)$ in both FID and accuracy, demonstrating the benefit of adaptive variance modeling. However, \textbf{P-Guide} is more sensitive to the choice of $w$ and exhibits a narrower effective range; for example, on CIFAR-10 performance degrades at $w=1.5$, while standard CFM remains stable and continues improving up to $w=3.0$. Overall, \textbf{P-Guide} achieves competitive performance within a limited $w$ range without incurring additional cost, whereas standard CFM attains better peak performance and robustness at the expense of doubled computation; we further show in Appendix~\ref{appendix:pg_cfg_compatibility} that the guidance scale in P-Guide preserves the essential behavior of CFG.

\begin{table*}[h]
  \caption{Quantitative results on CIFAR-10 and MNIST. All P-Guide results are obtained via a \textbf{single inference pass}. GFLOPs are reported per image generation (50 steps). FID and IS retain 3 significant figures; sFID retains 2 significant figures. Acc denotes classification accuracy (\%).}
  \label{tab:low_res_exhaustive}
  \centering
  \small
  \setlength{\tabcolsep}{2.pt}
  \begin{tabular}{l c c c c c c c c c c c c}
    \toprule
    & & & \multicolumn{5}{c}{CIFAR-10} & \multicolumn{5}{c}{MNIST} \\
    \cmidrule(r){4-8} \cmidrule(r){9-13}
    Method & Pass & $w$ 
    & FID $\downarrow$ & sFID $\downarrow$ & IS $\uparrow$ & Acc (\%) $\uparrow$ & GFLOPs $\downarrow$ 
    & FID $\downarrow$ & sFID $\downarrow$ & IS $\uparrow$ & Acc (\%) $\uparrow$ & GFLOPs $\downarrow$ \\
    \midrule
     CFM & 1 & 1.0 & 9.44 & 0.0063 & 8.61 & 70.53 & 710.4337  & 5.87 & 0.0061 & 2.10 & 95.468 & 417.5978 \\
    \gc{CFM} & \gc{2} & \gc{1.5} & \gc{4.00} & \gc{0.0017} & \gc{9.52} & \gc{86.07} & \gc{1420.8641} & \gc{5.94} & \gc{0.0061} & \gc{2.10} & \gc{99.608} & \gc{835.1956} \\
    \gc{CFM} & \gc{2} & \gc{2.0} & \gc{2.54} & \gc{0.0004} & \gc{10.10} & \gc{92.34} & \gc{1420.8641} & \gc{6.12} & \gc{0.0066} & \gc{2.09} & \gc{99.938} & \gc{835.1956} \\
    \gc{CFM} & \gc{2} & \gc{3.0} & \gc{3.72} & \gc{0.0009} & \gc{10.42} & \gc{96.63} & \gc{1420.8641} & \gc{11.5} & \gc{0.0120} & \gc{2.06} & \gc{99.982} & \gc{835.1956} \\
    \midrule
    PG($\sigma_\phi(y)$) & 1 & 0.9 & 9.22 & 0.0049 & 8.82 & 80.78 & 713.3849 & 2.39 & 0.0015 & 2.11 & 97.270 & 417.6066 \\
    PG($\sigma_\phi(y)$) & 1 & 1.0 
    & \cellcolor{gray!15}\textbf{6.51} & \cellcolor{gray!15}\textbf{0.0026} &  \cellcolor{gray!15}\textbf{9.07} & 83.99 & 711.9093 
    & \cellcolor{gray!15}\textbf{1.40} & \cellcolor{gray!15}\textbf{0.0007} & 2.11 & 98.974 & 417.6022 \\
    PG($\sigma_\phi(y)$) & 1 & 1.1 & 10.48 & 0.0041 & 9.03 & \cellcolor{gray!15}\textbf{86.00} & 713.3849 & 2.18 & 0.0011 & 2.09 & 99.520 & 417.6066 \\
    PG($\sigma_\phi(y)$) & 1 & 1.5 &  38.46 &  0.0164 & 7.65 & 76.75 & 713.3849 & 13.8 & 0.0098 & 2.02 & \cellcolor{gray!15}\textbf{99.962} & 417.6066 \\
    \midrule
    PG($\sigma=1$) & 1 & 0.9 & 12.59 & 0.0081 &  8.32 & 68.53 & 713.3849 & 3.21 & 0.0029 &  \cellcolor{gray!15}\textbf{2.13} & 97.480 & 417.6066 \\
    PG($\sigma=1$) & 1 & 1.0 & 9.35&  0.0057 &  8.61 & 72.49 & 711.9093 & 3.31 & 0.0030 & \cellcolor{gray!15}\textbf{2.13} & 98.216 & 417.6022 \\
    PG($\sigma=1$) & 1 & 1.1 & 8.72 &  0.0044&  8.67 & 76.69 & 713.3849 & 3.61& 0.0032 & 2.12 & 98.606 &417.6066 \\
    PG($\sigma=1$) & 1 & 1.5 & 18.23 &  0.0060&  8.69 & 81.12 & 713.3849 & 6.73 & 0.0055 & 2.08 & 99.510 & 417.6066 \\
    
    \bottomrule
  \end{tabular}
\end{table*}

\subsection{Scalability on ImageNet-1k}
\begin{figure*}[h]
  \centering
  \scalebox{0.95}{
  \begin{minipage}{0.48\textwidth}
    \centering
    \includegraphics[width=\linewidth]{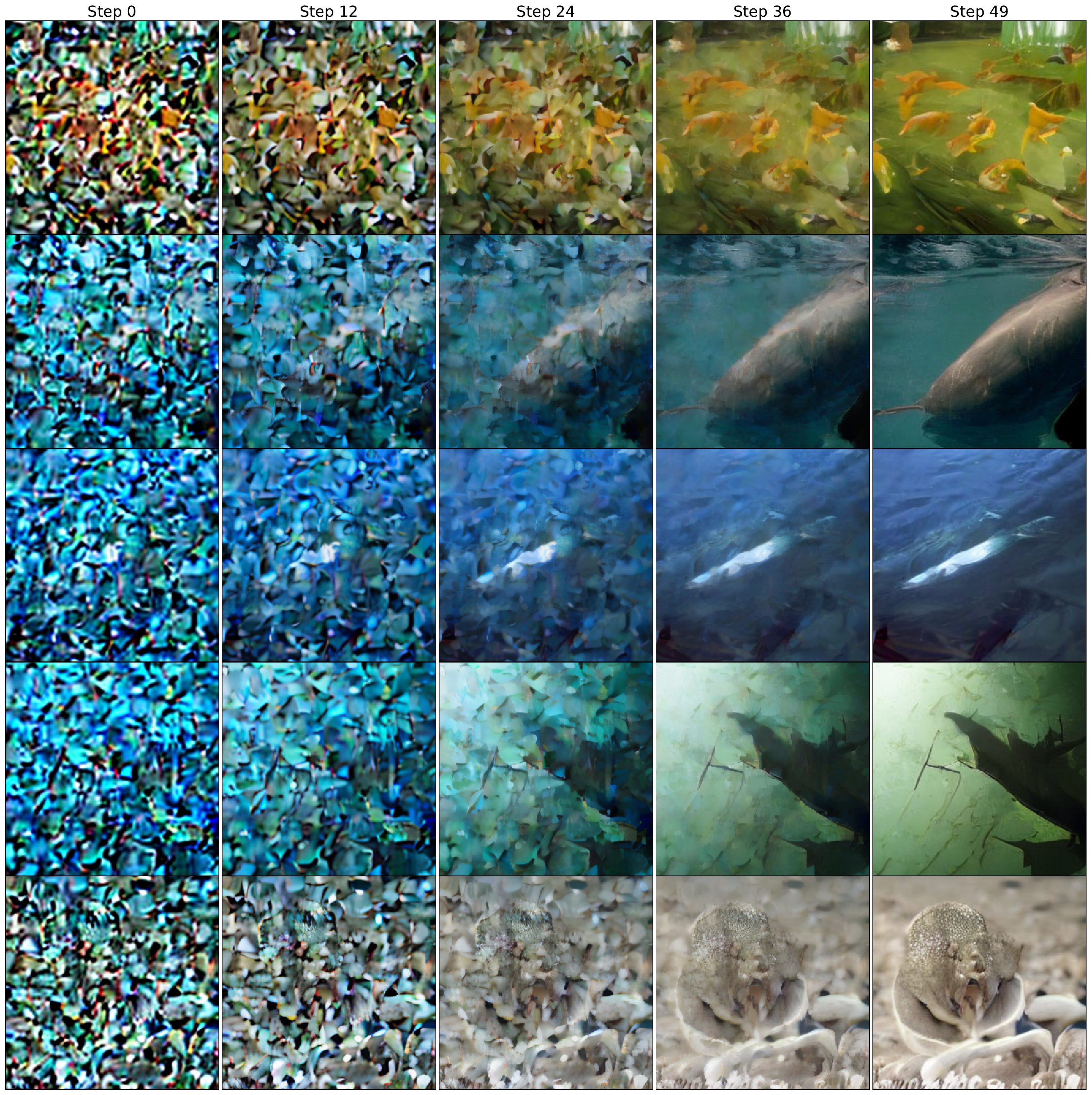}
    \small (a) Trajectory from $z_{cfg}$ to $x_1$ with $w=1.0$
  \end{minipage}
  \hfill
  \begin{minipage}{0.48\textwidth}
    \centering
    \includegraphics[width=\linewidth]{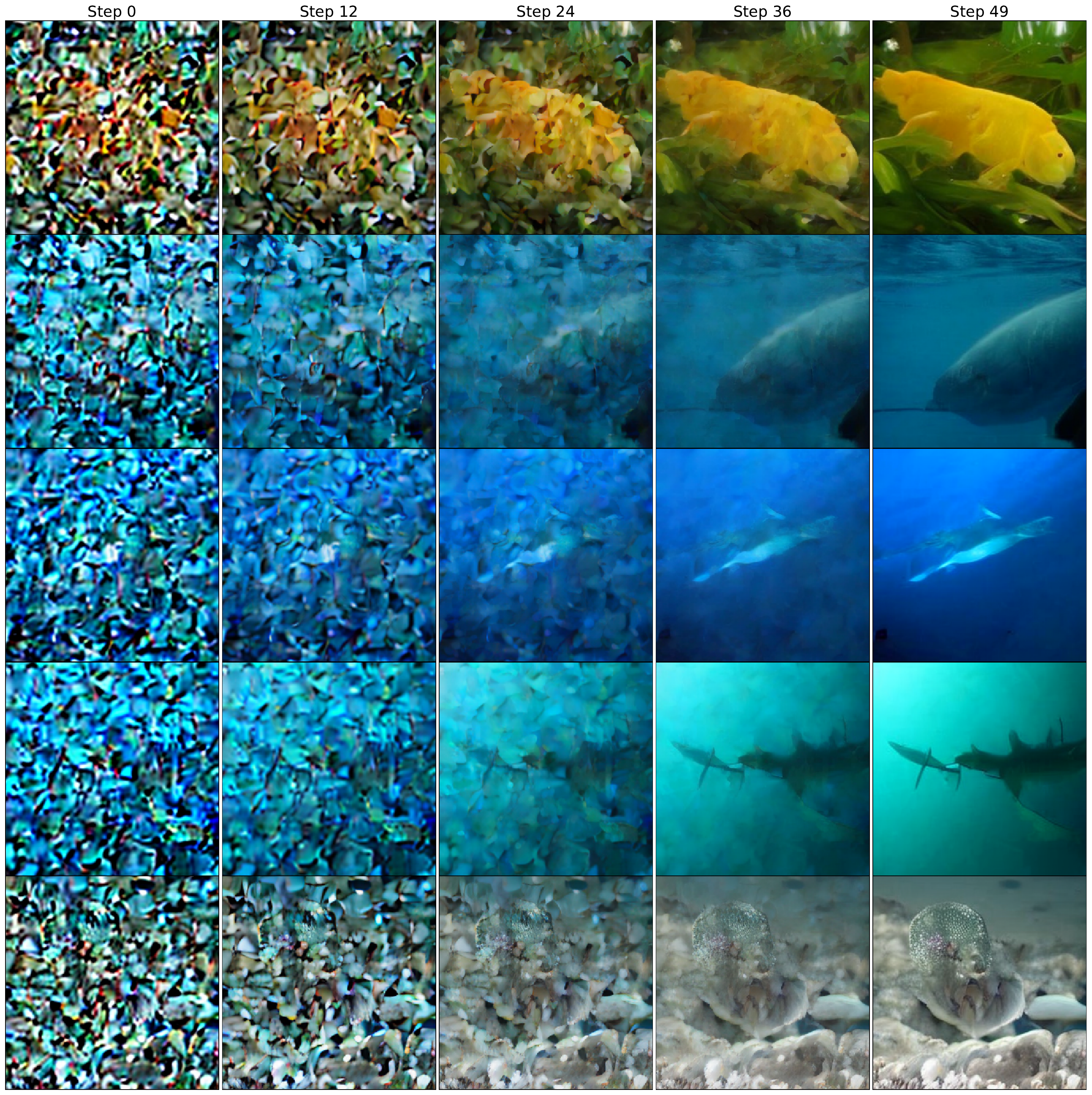}
    \small (b) Trajectory from $z_{cfg}$ to $x_1$ with $w=1.1$
  \end{minipage}}
\caption{\textbf{Visual comparison of P-Guide trajectories on ImageNet.} Each row shows the evolution from prior ($t=0$) to data ($t=1$). Increasing the guidance scale ($w=1.0 \rightarrow 1.1$) sharpens semantic structure from early steps, confirming effective trajectory-level control from the origin.}
\label{fig:ImageNet_trajs}
\end{figure*}
To assess scalability to high-resolution manifolds ($256 \times 256$), we evaluate P-Guide using both a \textbf{U-Net}~\cite{ronneberger2015u,ho2020ddpm} and a \textbf{Diffusion Transformer (DiT-B/2)}\cite{peebles2023dit}. To reduce computational costs, the Stage-2 flow model is initialized with weights from a pre-trained CFM model~\cite{dao2023flow} (originally trained for $\sim$10M steps) and further optimized for only 400K steps, requiring approximately 4\% of the original training budget.

\paragraph{Architecture Comparison.}
Table~\ref{tab:imagenet_exhaustive} highlights that P-Guide introduces only a minimal computational overhead while maintaining a single-pass inference paradigm across both architectures. As evidenced by the \textbf{Size} column, this improvement is achieved by adding merely \textbf{1.247 MB} of conditional parameters, corresponding to a negligible storage increase ($<0.5\%$ for U-Net).
Crucially, this efficient adaptation is enabled by our heteroscedastic modeling (Stage 1), which provides \textbf{Adaptive Loss Attenuation} \cite{kendall2017uncertainties}. This mechanism stabilizes optimization during prior learning and facilitates fast convergence to a high-quality conditional representation, even under limited fine-tuning budgets.

\paragraph{Discussion on Sampling Quality Improvement.} Similar to the original Classifier-Free Guidance theory\cite{cfgctrl2024}, P-Guide demonstrates the ability to improve generation fidelity by adjusting the guidance scale $w$ within a single inference pass. As shown in Table~\ref{tab:imagenet_exhaustive}, increasing $w$ from 1.0 to a moderate range (e.g., 1.1 for U-Net and 1.2 for DiT-B/2) leads to a clear reduction in FID, indicating that prior-space steering effectively sharpens the sample distribution. For example, the U-Net backbone FID improves from 27.78 to 22.33 at $w=1.1$, while DiT-B/2 reaches its best performance of 25.07 at $w=1.2$. We observe that the heteroscedastic and fixed-variance variants achieve comparable performance across different $w$, suggesting that both formulations provide similar guidance effects in this regime. However, further increasing $w$ leads to degradation (e.g., $w=1.5$), indicating a limited effective range of prior-space guidance. Overall, these results confirm that P-Guide can reproduce the quality improvement behavior of CFG within a single-pass framework, while maintaining constant computational cost.

\begin{table}[h]
  \caption{ImageNet-1k ($256 \times 256$) performance across various guidance scales. All P-Guide results utilize a \textbf{single inference pass}. Model size is reported as backbone + PG module (AutoKL 79.779MB excluded). GFLOPs are reported per image generation (50 steps).}
  \label{tab:imagenet_exhaustive}
  \centering
  \small
  \setlength{\tabcolsep}{3.0pt}
  \begin{tabular}{l cccc cccc}
    \toprule
    & \multicolumn{4}{c}{U-Net + \textbf{P-Guide}} & \multicolumn{4}{c}{DiT-B/2 + \textbf{P-Guide}} \\
    \cmidrule(lr){2-5} \cmidrule(lr){6-9}
    Scale $w$ 
    & \multicolumn{2}{c}{FID $\downarrow$} & GFLOPs $\downarrow$ & Size (MB) 
    & \multicolumn{2}{c}{FID $\downarrow$} & GFLOPs $\downarrow$ & Size (MB) \\
    \cmidrule(lr){2-3} \cmidrule(lr){6-7}
    & $\sigma=1$ & $\sigma_{\phi}(y)$ & & & $\sigma=1$ & $\sigma_{\phi}(y)$ & & \\
    \midrule
    1.0 & 28.40 & 27.78 & 3654.3964 & 388.55 + 1.247 & 32.84 & 33.68 & 2301.60 & 124.45 + 1.247 \\
    1.1 & 22.65 & \cellcolor{gray!15}\textbf{22.33} & 3655.8744 & 388.55 + 1.247 & 26.41 & 26.80 & 2303.08 & 124.45 + 1.247 \\
    1.2 & 22.67 & 23.01 & 3655.8744 & 388.55 + 1.247 & \cellcolor{gray!15}\textbf{24.61} & 25.07 & 2303.08 & 124.45 + 1.247 \\
    1.5 & 43.38 & 44.53 & 3655.8744 & 388.55 + 1.247 & 38.72 & 38.43 & 2303.08 & 124.45 + 1.247 \\
    \bottomrule
  \end{tabular}
\end{table}
\section{Conclusion}

We have presented P-Guide, a principled and parameter-efficient framework for single-pass conditional inference in Flow Matching. By shifting the guidance mechanism from iterative velocity-field extrapolation to the trajectory origin, P-Guide effectively mitigates the dual-pass computational bottleneck, achieving a $2\times$ throughput acceleration with negligible overhead. Our theoretical derivation establishes the trajectory-level approximation of prior-space steering, while heteroscedastic modeling provides robustness via adaptive loss attenuation. Crucially, P-Guide remains fully compatible with standard CFG (see Appendix~\ref{appendix:pg_cfg_compatibility}), acting as a complementary enhancement that preserves the effectiveness of traditional guidance signals without disrupting the underlying flow dynamics. We hope our work provides a scalable foundation for high-resolution generative systems and encourages further exploration of source-space control in continuous-time generative models.

\bibliographystyle{plainnat}  
\small
\bibliography{Reference}
\normalsize


\newpage

\appendix

\section{Proof of trajectory-level approximation}
\label{sec:appendix_equivalence}

In this section, we provide a detailed derivation showing that classifier-free guidance (CFG) applied in the prior space induces, to first order, the same trajectory perturbation as guidance performed directly in the velocity field. Throughout this appendix, we use the following notation:
\begin{itemize}
    \item $z \in \mathbb{R}^d$ denotes the initial latent/prior state,
    \item $\Phi_t(z)$ denotes the flow map of the ODE at time $t$,
    \item $v(x,t;z)$ denotes the velocity field induced by the initial state $z$,
    \item $\epsilon \sim \mathcal{N}(0,I)$ denotes the shared base noise used to couple conditional and unconditional trajectories.
\end{itemize}

\subsection{Preliminaries}
\label{subsec:prelims}

Let $x \sim p(x)$ be the data distribution and let $y$ denote the condition. The unconditional score is
\begin{equation}
    s_t(x) = \nabla_x \log p_t(x),
\end{equation}
and the conditional score is
\begin{equation}
    s_t(x \mid y) = \nabla_x \log p_t(x \mid y).
\end{equation}
By Bayes' rule,
\begin{equation}
    \label{eq:score_bayes_appendix}
    s_t(x \mid y) - s_t(x) = \nabla_x \log p_t(y \mid x).
\end{equation}
Hence, standard score-space CFG can be written as
\begin{equation}
    \label{eq:cfg_score_def_appendix}
    s_{\mathrm{cfg}}(x \mid y)
    = s_t(x) + w \bigl(s_t(x \mid y) - s_t(x)\bigr)
    = s_t(x) + w \nabla_x \log p_t(y \mid x).
\end{equation}

\subsection{Flow Matching as a Conditioned Dynamical System}
\label{subsec:fm_ode}

In Flow Matching, generation is described by an ODE
\begin{equation}
    \label{eq:fm_ode_appendix}
    \frac{d x_t}{d t} = v_\theta(x_t,t;z),
    \qquad x_0 = z.
\end{equation}
We denote the corresponding flow map by
\begin{equation}
    x_t = \Phi_t(z).
\end{equation}
Accordingly, the velocity along the trajectory is
\begin{equation}
    v(x_t,t;z) = \frac{d}{dt}\Phi_t(z).
\end{equation}

We consider two initial states constructed by the prior module:
\begin{equation}
    \label{eq:prior_defs_appendix}
    z_c = \mu_\phi(y) + \sigma_\phi(y)\odot \epsilon,
    \qquad
    z_u = \mu_\phi(\emptyset) + \sigma_\phi(\emptyset)\odot \epsilon,
\end{equation}
where the same noise $\epsilon$ is used in both branches. This shared-noise coupling is essential: it allows us to compare conditional and unconditional trajectories as two perturbations around a common random base point.

\subsection{Main Assumptions}
\label{subsec:assumptions}

We make the following mild assumptions.

\paragraph{Assumption 1 (Regular flow map).}
For each fixed $t$, the map $\Phi_t:\mathbb{R}^d \to \mathbb{R}^d$ is twice continuously differentiable in $z$.

\paragraph{Assumption 2 (Shared-noise coupling).}
Conditional and unconditional prior states are coupled by the same base noise $\epsilon$, as in Eq.~\eqref{eq:prior_defs_appendix}.

\paragraph{Assumption 3 (Local perturbation regime).}
We consider a first-order approximation regime with respect to the prior shift
\begin{equation}
    \Delta z \coloneqq z_c - z_u,
\end{equation}
where higher-order terms in the Taylor expansion around $z_u$ are neglected.

\paragraph{Assumption 4 (Optimal prior regression).}
The prior module is trained by Gaussian NLL, so that its mean prediction is the Bayes estimator under the model class:
\begin{equation}
    \mu_\phi(y) \approx \mathbb{E}[x \mid y],
    \qquad
    \mu_\phi(\emptyset) \approx \mathbb{E}[x].
\end{equation}
When the input is a noisy state $x_t$, the analogous conditional mean is $\mathbb{E}[x_0 \mid x_t,y]$.

\subsection{A Useful Linearization Lemma}
\label{subsec:linearization_lemma}

\begin{lemma}[Local linear response of the flow map]
\label{lem:linear_response}
Under Assumption 1 and Assumption 3, for every fixed $t$,
\begin{equation}
    \Phi_t(z_c)
    =
    \Phi_t(z_u) + J_t(z_u)\,\Delta z + R_t,
\end{equation}
where
\begin{equation}
    J_t(z_u) \coloneqq \nabla_z \Phi_t(z)\big|_{z=z_u}
\end{equation}
is the flow Jacobian, and the remainder satisfies
\begin{equation}
    \|R_t\| \le C_t \|\Delta z\|^2
\end{equation}
for some constant $C_t$ depending on the local curvature of $\Phi_t$.
\end{lemma}

\begin{proof}
This is the standard second-order Taylor expansion of a $C^2$ map around $z_u$. Since $\Phi_t$ is twice continuously differentiable in $z$, we have
\begin{equation}
    \Phi_t(z_c)
    =
    \Phi_t(z_u)
    +
    \nabla_z \Phi_t(z_u)\,(z_c-z_u)
    +
    O(\|z_c-z_u\|^2).
\end{equation}
Letting $\Delta z = z_c-z_u$ gives the claim.
\end{proof}

\subsection{Velocity Perturbation Induced by Prior Shift}
\label{subsec:velocity_shift}

\begin{lemma}[First-order velocity difference]
\label{lem:velocity_diff}
Under Assumption 1--3,
\begin{equation}
    v(x_t,t;z_c) - v(x_t,t;z_u)
    =
    \frac{d}{dt}\bigl(\Phi_t(z_c)-\Phi_t(z_u)\bigr)
    =
    \dot{J}_t(z_u)\,\Delta z + \tilde{R}_t,
\end{equation}
where $\|\tilde{R}_t\| = O(\|\Delta z\|^2)$.
\end{lemma}

\begin{proof}
By definition,
\begin{equation}
    v(x_t,t;z) = \frac{d}{dt}\Phi_t(z).
\end{equation}
Taking the difference between the conditional and unconditional trajectories,
\begin{equation}
    v(x_t,t;z_c)-v(x_t,t;z_u)
    =
    \frac{d}{dt}\Big(\Phi_t(z_c)-\Phi_t(z_u)\Big).
\end{equation}
Using Lemma~\ref{lem:linear_response},
\begin{equation}
    \Phi_t(z_c)-\Phi_t(z_u)
    =
    J_t(z_u)\Delta z + R_t.
\end{equation}
Differentiating with respect to $t$ yields
\begin{equation}
    v(x_t,t;z_c)-v(x_t,t;z_u)
    =
    \dot{J}_t(z_u)\Delta z + \tilde{R}_t,
\end{equation}
and $\tilde{R}_t = O(\|\Delta z\|^2)$ under the smoothness assumption.
\end{proof}

\subsection{Trajectory-Level First-Order Approximation}
\label{subsec:trajectory_equivalence}

\begin{theorem}[First-order trajectory approximation under shared-noise coupling]
\label{thm:trajectory_equivalence}
Let $z_c$ and $z_u$ be defined by Eq.~\eqref{eq:prior_defs_appendix}. Under Assumption 1--3, the conditional and unconditional trajectories satisfy
\begin{equation}
    x_t^{(c)} - x_t^{(u)}
    =
    J_t(z_u)\,(z_c-z_u) + O(\|z_c-z_u\|^2),
\end{equation}
and their velocity difference satisfies
\begin{equation}
    v_t^{(c)} - v_t^{(u)}
    =
    \dot{J}_t(z_u)\,(z_c-z_u) + O(\|z_c-z_u\|^2).
\end{equation}

Consequently, both prior-space perturbations and velocity-space differences are driven by the same first-order direction $\Delta z = z_c - z_u$. In particular, for
\begin{equation}
    z_{\mathrm{cfg}} = z_u + w(z_c - z_u),
\end{equation}
we obtain the induced trajectory perturbation
\begin{equation}
    x_t^{\mathrm{cfg}}
    =
    x_t^{(u)} + w\,J_t(z_u)(z_c-z_u) + O(\|z_c-z_u\|^2).
\end{equation}

Moreover, the velocity-space CFG formulation
\begin{equation}
    v_{\mathrm{cfg}}(x_t,t)
    =
    v_t^{(u)}(x_t,t) + \tilde{w}\bigl(v_t^{(c)}(x_t,t)-v_t^{(u)}(x_t,t)\bigr)
\end{equation}
induces a trajectory correction that, under a first-order linearization of the flow map, aligns with the same perturbation direction $J_t(z_u)\Delta z$, up to scaling differences between $w$ and $\tilde{w}$ arising from the nonlinear state-to-velocity mapping.

\end{theorem}

\begin{proof}
By Lemma~\ref{lem:linear_response},
\begin{equation}
    x_t^{(c)}-x_t^{(u)}
    =
    \Phi_t(z_c)-\Phi_t(z_u)
    =
    J_t(z_u)(z_c-z_u) + O(\|z_c-z_u\|^2).
\end{equation}

Similarly, by Lemma~\ref{lem:velocity_diff},
\begin{equation}
    v_t^{(c)}-v_t^{(u)}
    =
    \dot{J}_t(z_u)(z_c-z_u) + O(\|z_c-z_u\|^2).
\end{equation}

The key observation is that both state and velocity differences are governed by the same latent-space perturbation direction $\Delta z$, although they are expressed in different coordinate systems (state space vs. velocity field space).

For the guided prior
\begin{equation}
    z_{\mathrm{cfg}} = z_u + w\Delta z,
\end{equation}
a first-order Taylor expansion yields
\begin{equation}
    x_t^{\mathrm{cfg}}
    =
    \Phi_t(z_u) + w J_t(z_u)\Delta z + O(\|\Delta z\|^2).
\end{equation}

For velocity-space CFG, note that the correction
\begin{equation}
    \delta v_t = \tilde{w}(v_t^{(c)} - v_t^{(u)})
\end{equation}
acts as an Eulerian perturbation of the vector field. The induced trajectory perturbation is governed by the variational dynamics
\begin{equation}
    \frac{d}{dt}\delta x_t = \nabla_x v(x_t,t)\,\delta x_t + \delta v_t.
\end{equation}

Under a first-order linearization of the flow map and the shared-noise coupling assumption, prior-space guidance and velocity-space CFG induce the same first-order perturbation direction in the linearized trajectory space spanned by $J_t(z_u)\Delta z$.

This implies that CFG can be equivalently implemented in the prior space, avoiding the need for dual forward passes required by velocity-space guidance.
\end{proof}

\subsection{Connection to the Conditional Score}
\label{subsec:score_connection}

To make the link to score-space CFG explicit, we connect the prior shift to the conditional score under a Gaussian perturbation interpretation.

\begin{proposition}[Prior shift and conditional score]
\label{prop:prior_score}
Assume a Gaussian perturbation model at noise level $\sigma_t$, under which the posterior mean admits a score-based representation. Then the difference between conditional and unconditional estimators satisfies
\begin{equation}
    \mathbb{E}[x_0 \mid x_t, y] - \mathbb{E}[x_0 \mid x_t]
    =
    \sigma_t^2 \Bigl(\nabla_{x_t}\log p_t(x_t \mid y) - \nabla_{x_t}\log p_t(x_t)\Bigr)
    =
    \sigma_t^2 \nabla_{x_t}\log p(y \mid x_t).
\end{equation}
\end{proposition}

\begin{proof}
Under a Gaussian perturbation model at noise level $\sigma_t$, the posterior mean can be expressed in score form as
\begin{equation}
    \mathbb{E}[x_0 \mid x_t]
    =
    x_t + \sigma_t^2 \nabla_{x_t}\log p_t(x_t),
\end{equation}
and similarly in the conditional case,
\begin{equation}
    \mathbb{E}[x_0 \mid x_t, y]
    =
    x_t + \sigma_t^2 \nabla_{x_t}\log p_t(x_t \mid y).
\end{equation}

Subtracting the two expressions yields
\begin{equation}
    \mathbb{E}[x_0 \mid x_t, y] - \mathbb{E}[x_0 \mid x_t]
    =
    \sigma_t^2\left(\nabla_{x_t}\log p_t(x_t\mid y)-\nabla_{x_t}\log p_t(x_t)\right).
\end{equation}

Using Bayes' rule,
\begin{equation}
    \nabla_{x_t}\log p_t(x_t \mid y)-\nabla_{x_t}\log p_t(x_t)
    =
    \nabla_{x_t}\log p(y \mid x_t),
\end{equation}
which completes the derivation.
\end{proof}

Combining Proposition~\ref{prop:prior_score} with Theorem~\ref{thm:trajectory_equivalence}, we interpret the prior-space shift induced by CFG as a first-order approximation to the conditional score direction. This provides a unified perspective linking prior-space guidance and score-based guidance under Gaussian perturbation, consistent with the trajectory-level equivalence established in Section~\ref{subsec:trajectory_equivalence}.

\subsection{Adaptive Loss Attenuation in Training}
\label{subsec:loss_attenuation}

We parameterize the prior module to output both mean and variance:
\begin{equation}
    \mu_\phi(y), \qquad \sigma_\phi^2(y) > 0.
\end{equation}
Assuming a Gaussian likelihood, the training objective is
\begin{equation}
    \label{eq:nll_loss_appendix}
    \mathcal{L}
    =
    \mathbb{E}_{x,y}\left[
    \frac{\|x-\mu_\phi(y)\|^2}{2\sigma_\phi^2(y)}
    +
    \frac{1}{2}\log \sigma_\phi^2(y)
    \right].
\end{equation}

\begin{theorem}[Adaptive loss attenuation]
\label{thm:loss_attenuation}
Under the NLL objective in Eq.~\eqref{eq:nll_loss_appendix}, the gradient of the mean prediction is attenuated by the inverse variance:
\begin{equation}
    \nabla_{\mu_\phi}\mathcal{L}
    =
    -\frac{x-\mu_\phi(y)}{\sigma_\phi^2(y)}.
\end{equation}
Therefore, large predictive uncertainty downweights the contribution of hard or noisy examples.
\end{theorem}

\begin{proof}
The terms depending on $\mu_\phi(y)$ are
\begin{equation}
    \frac{1}{2\sigma_\phi^2(y)}\|x-\mu_\phi(y)\|^2.
\end{equation}
Differentiating with respect to $\mu_\phi(y)$ gives
\begin{equation}
    \nabla_{\mu_\phi}\mathcal{L}
    =
    \frac{1}{2\sigma_\phi^2(y)} \cdot 2(\mu_\phi(y)-x)
    =
    -\frac{x-\mu_\phi(y)}{\sigma_\phi^2(y)}.
\end{equation}
Hence the gradient magnitude is inversely proportional to $\sigma_\phi^2(y)$, which realizes adaptive attenuation.
\end{proof}

\subsection{Stability of Velocity-Space Control}
\label{subsec:stability_analysis}

A practical reason for operating in velocity space is that the ODE is controlled directly at the level that determines the trajectory evolution. Consider a guided vector field
\begin{equation}
    v_{\mathrm{cfg}}(x,t)
    =
    v_u(x,t) + w\bigl(v_c(x,t)-v_u(x,t)\bigr).
\end{equation}
If $v_u$ and $v_c$ are Lipschitz in $x$, then $v_{\mathrm{cfg}}$ is also Lipschitz, and the ODE remains well-posed. Moreover, by Gr\"onwall's inequality, the trajectory deviation is bounded by
\begin{equation}
    \|x_t^{\mathrm{cfg}} - x_t^{(u)}\|
    \le
    e^{Lt}\int_0^t \|w(v_c(x_s,s)-v_u(x_s,s))\|\,ds,
\end{equation}
where $L$ is a Lipschitz constant. This shows that velocity-space guidance provides a continuous and stable control signal over the entire integration path.

By contrast, directly mixing parameterizations such as means and variances changes the initial distribution in a way that need not preserve the shared-noise coupling structure. In particular, parameter interpolation does not in general correspond to a controlled perturbation of the vector field and therefore is not trajectory-equivalent to CFG in the ODE dynamics.

\subsection{Final Unified Statement}
\label{subsec:final_statement}

We summarize the above results in the following compact formula. Under the denoising interpretation of the prior module and the local linear response of the flow map, the guided prior
\begin{equation}
    \label{eq:p_guide_core_final_appendix}
    z_{\mathrm{cfg}}
    =
    \mu_\phi(\emptyset)
    +
    w\bigl(\mu_\phi(y)-\mu_\phi(\emptyset)\bigr)
    +
    \Bigl[
    \sigma_\phi(\emptyset)
    +
    w\bigl(\sigma_\phi(y)-\sigma_\phi(\emptyset)\bigr)
    \Bigr]\odot \epsilon
\end{equation}
induces, to first order, the same trajectory perturbation as the corresponding velocity-space CFG. Therefore, a single guidance operation at initialization can replace repeated conditional/unconditional evaluations along the full trajectory, while preserving the leading-order guidance direction in the dynamical system.

\qed

\section{Closed-Form CFG in Distribution Space and Failure Analysis}
\label{sec:appendix_distribution_cfg}

\subsection{Closed-Form CFG in Distribution Space}

Consider the unconditional and conditional initial distributions:
\begin{equation}
	p_u(x) = \mathcal{N}(\mu_u, \sigma_u^2 \mathbf{I}), \quad 
	p_c(x) = \mathcal{N}(\mu_c, \sigma_c^2 \mathbf{I}).
\end{equation}

A natural distribution-level formulation of CFG is:
\begin{equation}
	p_{\mathrm{cfg}}(x) \propto p_u(x)^{1-w} p_c(x)^w,
\end{equation}
which corresponds to a linear interpolation in log-density space:
\begin{equation}
	\log p_{\mathrm{cfg}}(x) = (1-w)\log p_u(x) + w\log p_c(x) + C.
\end{equation}

Substituting the Gaussian log-density
\begin{equation}
	\log p(x) = -\frac{1}{2\sigma^2}\|x - \mu\|^2 - \frac{1}{2}\log(2\pi\sigma^2),
\end{equation}
and ignoring constants independent of $x$, we obtain:
\begin{equation}
	\log p_{\mathrm{cfg}}(x)
	= -\frac{1-w}{2\sigma_u^2}\|x - \mu_u\|^2 
	- \frac{w}{2\sigma_c^2}\|x - \mu_c\|^2 + C.
\end{equation}

We now expand both quadratic terms explicitly:
\begin{align}
	\|x - \mu_u\|^2 &= \|x\|^2 - 2\mu_u^\top x + \|\mu_u\|^2, \\
	\|x - \mu_c\|^2 &= \|x\|^2 - 2\mu_c^\top x + \|\mu_c\|^2.
\end{align}

Substituting back:
\begin{align}
	\log p_{\mathrm{cfg}}(x)
	&= -\frac{1-w}{2\sigma_u^2}(\|x\|^2 - 2\mu_u^\top x + \|\mu_u\|^2) \nonumber \\
	&\quad - \frac{w}{2\sigma_c^2}(\|x\|^2 - 2\mu_c^\top x + \|\mu_c\|^2) + C.
\end{align}

Grouping terms by powers of $x$:
\begin{align}
	\log p_{\mathrm{cfg}}(x)
	&= -\frac{1}{2}\left( \frac{1-w}{\sigma_u^2} + \frac{w}{\sigma_c^2} \right)\|x\|^2 \nonumber \\
	&\quad + \left( \frac{1-w}{\sigma_u^2}\mu_u + \frac{w}{\sigma_c^2}\mu_c \right)^\top x + C'.
\end{align}

Matching this with the canonical Gaussian form
\begin{equation}
	-\frac{1}{2\sigma_{\mathrm{cfg}}^2}\|x - \mu_{\mathrm{cfg}}\|^2,
\end{equation}
we obtain:
\begin{equation}
\label{eq:p-sigma}
	\frac{1}{\sigma_{\mathrm{cfg}}^2} = \frac{1-w}{\sigma_u^2} + \frac{w}{\sigma_c^2},
\end{equation}
\begin{equation}
\label{eq:p-mu}
	\mu_{\mathrm{cfg}} = \sigma_{\mathrm{cfg}}^2 \left(
	\frac{1-w}{\sigma_u^2}\mu_u + \frac{w}{\sigma_c^2}\mu_c
	\right).
\end{equation}

Thus, CFG in distribution space results in a Gaussian whose mean and variance are jointly modified.

In the special case where $\sigma_u = \sigma_c = \sigma$, the above expressions simplify to:
\begin{equation}
    \frac{1}{\sigma_{\mathrm{cfg}}^2} = \frac{1}{\sigma^2}, \quad
    \mu_{\mathrm{cfg}} = (1-w)\mu_u + w\mu_c.
\end{equation}
This recovers the standard linear interpolation in mean space, consistent with the distribution-level CFG formulation in Eq.~\ref{eq:p-sigma}--\ref{eq:p-mu}.

Moreover, this form matches the marginal distribution induced by P-Guide in Eq.~\ref{eq:p_guide_core_final_appendix}, where the initial latent distribution is constructed as a convex combination of class-dependent components.

Therefore, the claim that distribution-level CFG is inherently inconsistent with ODE-based dynamics only holds in the general heteroscedastic case, and does not contradict the theoretical validity of the proposed P-Guide formulation under the homoscedastic setting.

\subsection{Why Distribution-Level CFG Is Problematic}

The key limitation of distribution-level interpolation is that it modifies the \textit{shape of the density}, rather than directly influencing the \textit{direction of trajectory evolution} induced by the underlying flow dynamics.

Sampling directly from $p_{\mathrm{cfg}}(x)$ effectively:
\begin{itemize}
	\item defines a new initial distribution,
	\item while keeping the dynamics $v(x_t, t \mid z)$ unchanged.
\end{itemize}

This leads to a mismatch between the sampling initialization and the velocity field. In particular, the shared-noise coupling assumption
\begin{equation}
	(z_c, z_u) \ \text{are no longer generated from the same base noise } \epsilon,
\end{equation}
which is a key requirement in Section~\ref{sec:appendix_equivalence}, may no longer hold.

As a consequence, the trajectory-level interpretation of CFG becomes less accurate in the general heteroscedastic setting, as the guidance signal is partially decoupled from the underlying dynamics.

\subsection{Empirical Failure on MNIST}

To empirically verify the mismatch identified in Section~\ref{sec:appendix_distribution_cfg}, we conduct a controlled experiment on MNIST by directly sampling from the distribution-level CFG formulation $p_{\mathrm{cfg}}(x)$ instead of applying trajectory-level guidance.

Figure~\ref{fig:mnist_cfg_compare} presents a qualitative comparison between trajectory-level P-Guide and distribution-level CFG under the same initial noise and guidance scale ($w=1.5$). 

We observe that, although the distribution-level formulation modifies the marginal density, it fails to produce meaningful improvements in sample quality. In particular, the generated digits exhibit noticeable distortions and structural inconsistencies, while P-Guide preserves clear and semantically correct digit shapes.

This behavior remains consistent across different digit classes, indicating that increasing the guidance scale does not lead to systematic improvements under the distribution-level formulation.

\begin{figure}[H]
    \centering
    \includegraphics[width=\linewidth]{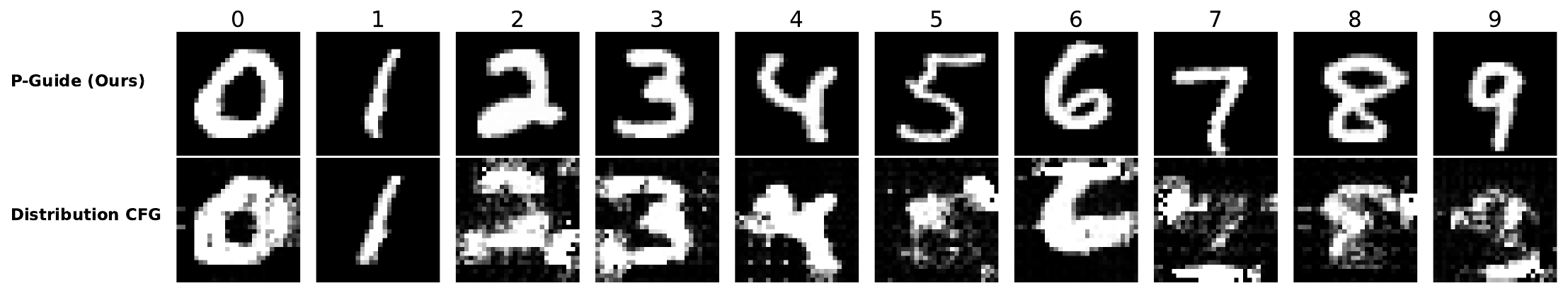}
    \caption{Comparison between trajectory-level guidance (P-Guide) and distribution-level CFG on MNIST under guidance scale $w=1.5$. The distribution-level formulation leads to distorted and unstable samples, while P-Guide preserves semantic structure.}
    \label{fig:mnist_cfg_compare}
\end{figure}

This failure is consistent with our theoretical analysis in Section~\ref{sec:appendix_equivalence}. Specifically, distribution-level CFG violates the shared-noise coupling assumption, which is essential for maintaining trajectory-level coherence between conditional and unconditional paths.

As a result, the geometric interpretation of CFG as a trajectory-level perturbation breaks down, and the method degenerates into a static density reshaping procedure that is not aligned with the underlying ODE dynamics.

These results further support the conclusion that effective guidance must operate at the trajectory level rather than directly modifying probability distributions.

\section{Additional Ablation Studies}
\label{appendix:ablations}

\subsection{Effect of Initialization: Pretrained vs. From-Scratch Training}
\label{appendix:ablation_pretrain}

We further investigate the impact of initialization on training efficiency for the Stage-2 flow model in P-Guide. Specifically, we compare two settings: (1) training the model from scratch, and (2) initializing from an existing pretrained flow model~\cite{dao2023flow}. In both cases, the model is optimized for 400K steps under identical training configurations.

All experiments are conducted using the DiT-B/2 backbone with a single inference pass. Quantitative results are summarized in Table~\ref{tab:pretrain_vs_scratch}.

\begin{table}[t]
\centering
\caption{Comparison between pretrained initialization and training from scratch for Stage-2 P-Guide. Both models are trained for 400K steps.}
\label{tab:pretrain_vs_scratch}
\begin{tabular}{lcc}
\toprule
Initialization & FID $\downarrow$ & GFLOPs $\downarrow$ \\
\midrule
From Scratch (400K) & 49.03 & 2301.60 \\
Pretrained + 400K   & 33.68 & 2301.60 \\
\bottomrule
\end{tabular}
\end{table}

We observe that P-Guide can be successfully trained from scratch, achieving reasonable performance under a 400K-step training budget. However, initializing from a pretrained flow model consistently leads to better sample quality under the same number of optimization steps, indicating faster convergence.

This suggests that while P-Guide does not rely on pretrained models, it can effectively leverage existing flow models as a strong initialization in practice. Such initialization provides a convenient way to accelerate training of the Stage-2 model without altering the inference procedure.

Overall, these results demonstrate that P-Guide is compatible with both from-scratch and pretrained settings, while offering improved training efficiency when combined with existing flow models.

\subsection{Influence of Prior Steering Module Capacity}
\label{appendix:ablation_param_scale}

To investigate whether increasing the capacity of the prior steering module $F_{\phi}(y)$ leads to further improvements, we compare the default lightweight design with an enlarged variant (PG-Large, 12.5 MB). The architecture of the default module (1.247 MB) is illustrated in Fig.~\ref{fig:pg_resnet}, which adopts a compact conditional ResNet with FiLM modulation.

\begin{figure}[t]
    \centering
    \includegraphics[width=\linewidth]{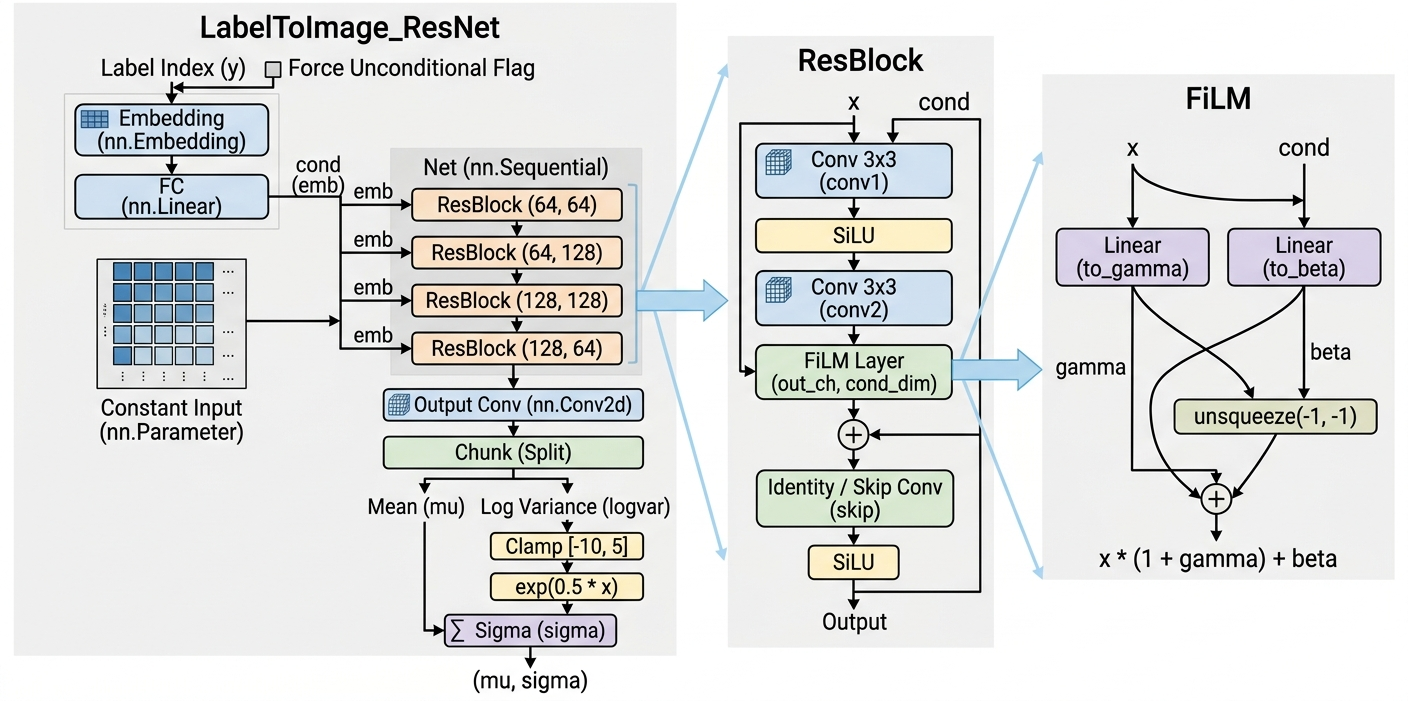}
    \caption{Architecture of the default P-Guide prior steering module (1.247 MB). The model consists of a lightweight conditional ResNet with FiLM modulation.}
    \label{fig:pg_resnet}
\end{figure}

The larger variant scales up the model capacity by increasing both the number of residual blocks and the hidden channel width, while keeping the overall design identical.

Quantitative results are reported in Table~\ref{tab:pg_capacity}. Despite a $10\times$ increase in parameter size, the performance remains comparable (FID changes from 33.68 to 34.34), while the computational cost remains nearly unchanged.

\begin{table}[t]
\centering
\caption{Effect of prior module capacity on ImageNet-1k ($256 \times 256$). All results use DiT-B/2 backbone with a single inference pass.}
\label{tab:pg_capacity}
\begin{tabular}{lccc}
\toprule
PG Module & Size (MB) & FID $\downarrow$ & GFLOPs $\downarrow$ \\
\midrule
P-Guide (Default) & 1.247 & 33.68 & 2301.60 \\
P-Guide (Large)   & 12.5  & 34.34 & 2303.65 \\
\bottomrule
\end{tabular}
\end{table}

This observation indicates that the prior shift $z_c - z_u$ mainly acts as a low-dimensional directional signal in trajectory space, rather than requiring a high-capacity function approximator. In other words, the role of $F_{\phi}(y)$ is not to model complex distributions, but to provide a coarse yet effective steering direction.

Therefore, a lightweight conditional network is sufficient to capture the essential structure needed for single-pass guidance, leading to a favorable trade-off between performance and efficiency.
\subsection{Effect of Applying Guidance to Variance}
\label{appendix:ablation_sigma_cfg}

We investigate whether classifier-free guidance (CFG) needs to be applied to both the mean and variance of the conditional prior in P-Guide. The default formulation applies guidance to both components:
\begin{equation}
    z_{cfg} = \mu_{\phi}(\emptyset) + w (\mu_{\phi}(y) - \mu_{\phi}(\emptyset)) 
    + \big[\sigma_{\phi}(\emptyset) + w (\sigma_{\phi}(y) - \sigma_{\phi}(\emptyset))\big] \odot \epsilon.
\end{equation}

We compare this with a simplified variant where CFG is applied only to the mean, while the variance remains conditional:
\begin{equation}
    z_{cfg} = \mu_{\phi}(\emptyset) + w (\mu_{\phi}(y) - \mu_{\phi}(\emptyset)) 
    + \sigma_{\phi}(y) \odot \epsilon.
\end{equation}

All experiments are conducted using both U-Net and DiT-B/2 backbones with a single inference pass. Quantitative results are reported in Table~\ref{tab:sigma_cfg}.

\begin{table}[t]
\centering
\caption{Comparison of applying CFG to mean only vs. both mean and variance. Results are reported in FID ($\downarrow$) across different guidance scales $w$.}
\label{tab:sigma_cfg}
\begin{tabular}{ccccc}
\toprule
& \multicolumn{2}{c}{U-Net} & \multicolumn{2}{c}{DiT-B/2} \\
\cmidrule(lr){2-3} \cmidrule(lr){4-5}
$w$ & Mean-only & Mean+Var & Mean-only & Mean+Var \\
\midrule
1.1 & 22.33 & 22.09 & 26.80 & 26.71 \\
1.2 & 23.01 & 23.14 & 25.07 & 24.75 \\
1.5 & 44.53 & 47.62 & 38.43 & 38.31 \\
\bottomrule
\end{tabular}
\end{table}

We observe that applying CFG to the mean alone already achieves competitive performance, while incorporating variance guidance provides only marginal improvements in most cases. This suggests that the primary effect of guidance is captured by the mean shift in the prior space.

These findings further support our interpretation that P-Guide operates mainly through a directional control mechanism in the latent space, where the mean term dominates trajectory steering, and variance plays a secondary role.

\subsection{Compatibility with Standard CFG}
\label{appendix:pg_cfg_compatibility}

To further examine whether P-Guide is compatible with standard classifier-free guidance (CFG), we conduct an additional experiment where both mechanisms are applied jointly. Specifically, we first perform prior steering using P-Guide, and then apply standard CFG during ODE integration.

We evaluate combinations of guidance scales for P-Guide ($w_{\text{PG}}$) and standard CFG ($w_{\text{CFG}}$). 
Specifically, $w_{\text{PG}} \in \{1.0, 1.1, 1.2\}$ and $w_{\text{CFG}} \in \{1.0, 1.1, 1.2, 1.5, 2.0\}$, resulting in 15 configurations in total. 
Results on ImageNet-1k ($256 \times 256$) with a U-Net backbone are reported in Table~\ref{tab:pg_cfg_joint}.

\begin{table}[t]
\centering
\caption{Joint evaluation of P-Guide and standard CFG on ImageNet-1k ($256 \times 256$). Rows correspond to P-Guide scale $w_{\text{PG}}$, columns correspond to CFG scale $w_{\text{CFG}}$. All results use a U-Net backbone.}
\label{tab:pg_cfg_joint}
\begin{tabular}{c|ccccc}
\toprule
$w_{\text{PG}} \backslash w_{\text{CFG}}$ 
& 1.0 & 1.1 & 1.2 & 1.5& 2.0 \\
\midrule
1.0 & 27.78 & 23.37 & 20.01 & 15.44&25.29 \\
1.1 & 22.33 & 18.74 & 16.16 & 13.26&21.00 \\
1.2 &  23.01 & 20.43 & 18.77 & 18.54& 29.35\\
\bottomrule
\end{tabular}
\end{table}

We observe that combining P-Guide with standard CFG does not lead to instability or degradation beyond the expected behavior of large guidance scales. In particular, moderate combinations (e.g., $w_{\text{PG}} \in [1.0,1.2]$, $w_{\text{CFG}} \in [1.2,2.0]$) remain stable and achieve competitive FID values.

These results suggest that P-Guide operates as a complementary mechanism to standard CFG, rather than interfering with its effect. This further supports our interpretation that prior-space steering and trajectory-level velocity extrapolation act on different stages of the generative process and can be composed without violating the underlying dynamics.

\section{Additional Qualitative Results}
\label{appendix:qualitative}

In this section, we provide additional qualitative results to complement the quantitative evaluations in the main text and Appendix~\ref{appendix:ablations}. These visualizations aim to further illustrate the behavior of P-Guide under different settings, including guidance strength, prior parameterization, and model architectures.

\begin{figure*}[h]
  \centering
  \begin{minipage}{0.48\textwidth}
    \centering
    \includegraphics[width=\linewidth]{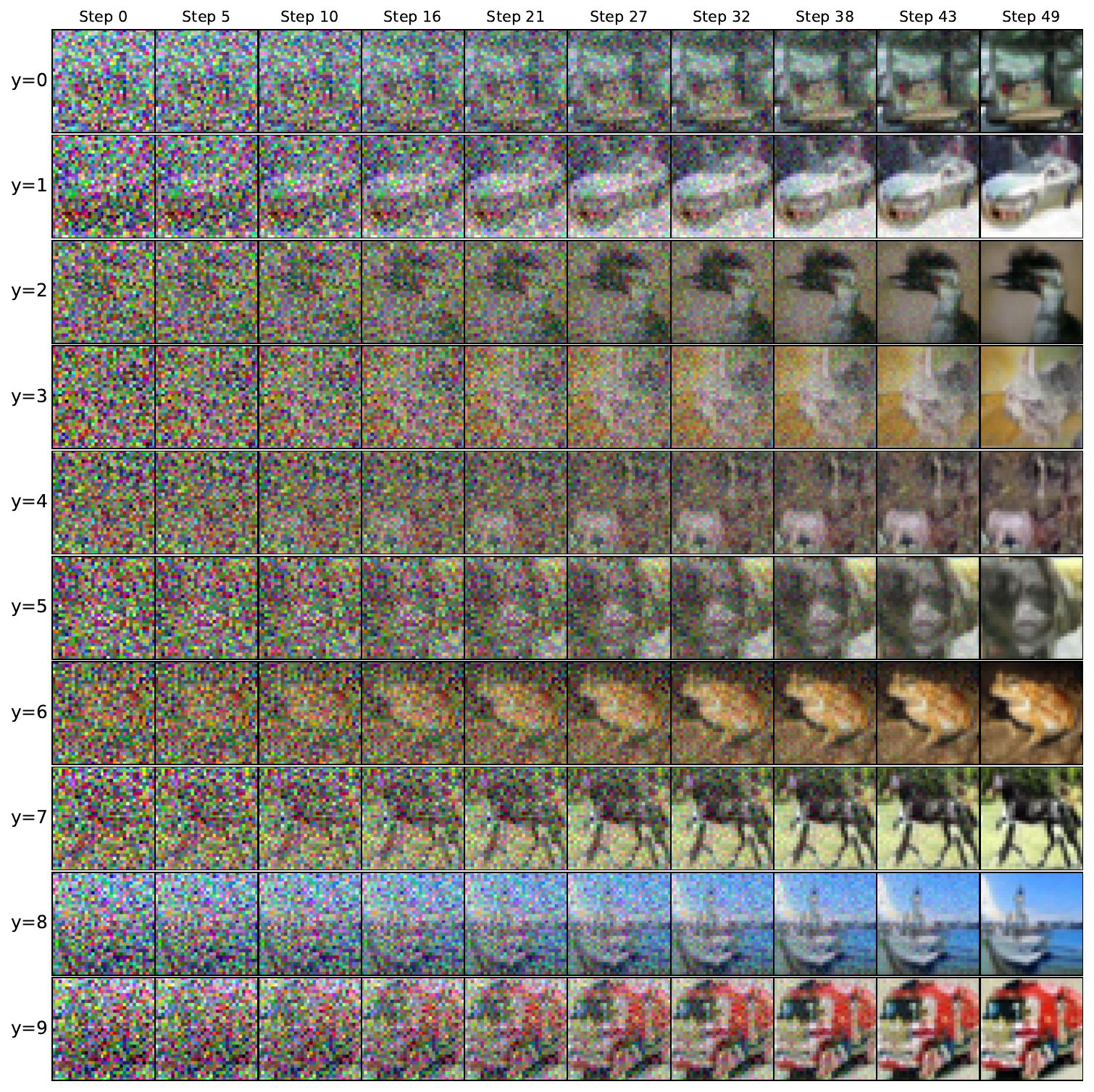}
    \small (a) Trajectory from $z_{cfg}$ to $x_1$ with $w=1.0$
  \end{minipage}
  \hfill
  \begin{minipage}{0.48\textwidth}
    \centering
    \includegraphics[width=\linewidth]{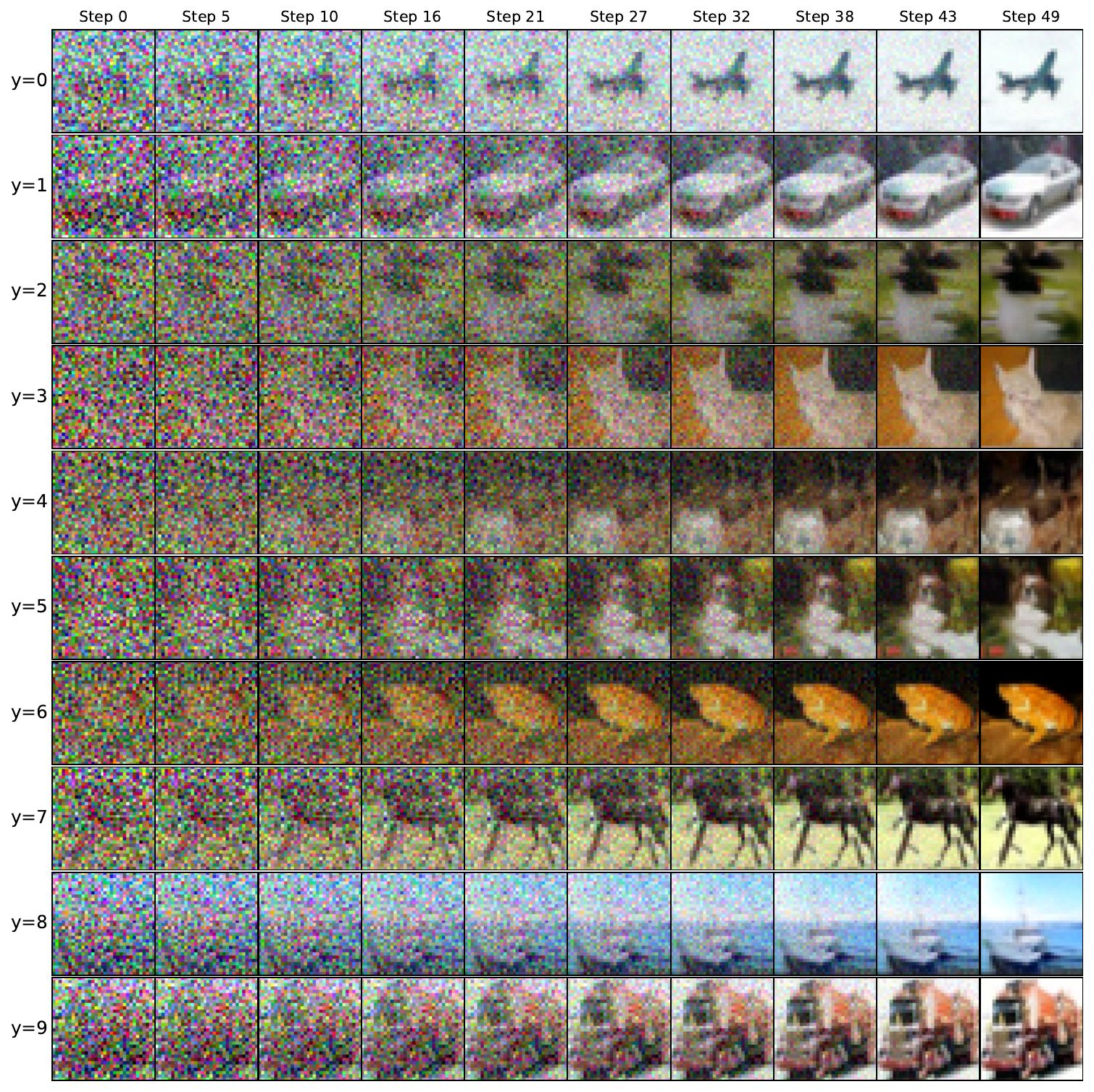}
    \small (b) Trajectory from $z_{cfg}$ to $x_1$ with $w=1.2$
  \end{minipage}
    \caption{\textbf{Visual comparison of P-Guide generation trajectories on CIFAR-10.} 
    Each row shows the evolution of a sample from its initial latent state ($t=0$) to the final generated image ($t=1$). 
    Comparing $w=1.0$ and $w=1.2$, increasing the guidance scale leads to more semantically coherent structures emerging from the earliest stages of generation. 
    This supports our hypothesis that modulating the initial latent state is sufficient to influence the entire sampling trajectory, without requiring iterative velocity-field extrapolation.}
    \label{fig:CIFAR10_trajs}
\end{figure*}

\begin{figure*}[h]
  \centering
  \begin{minipage}{0.48\textwidth}
    \centering
    \includegraphics[width=\linewidth]{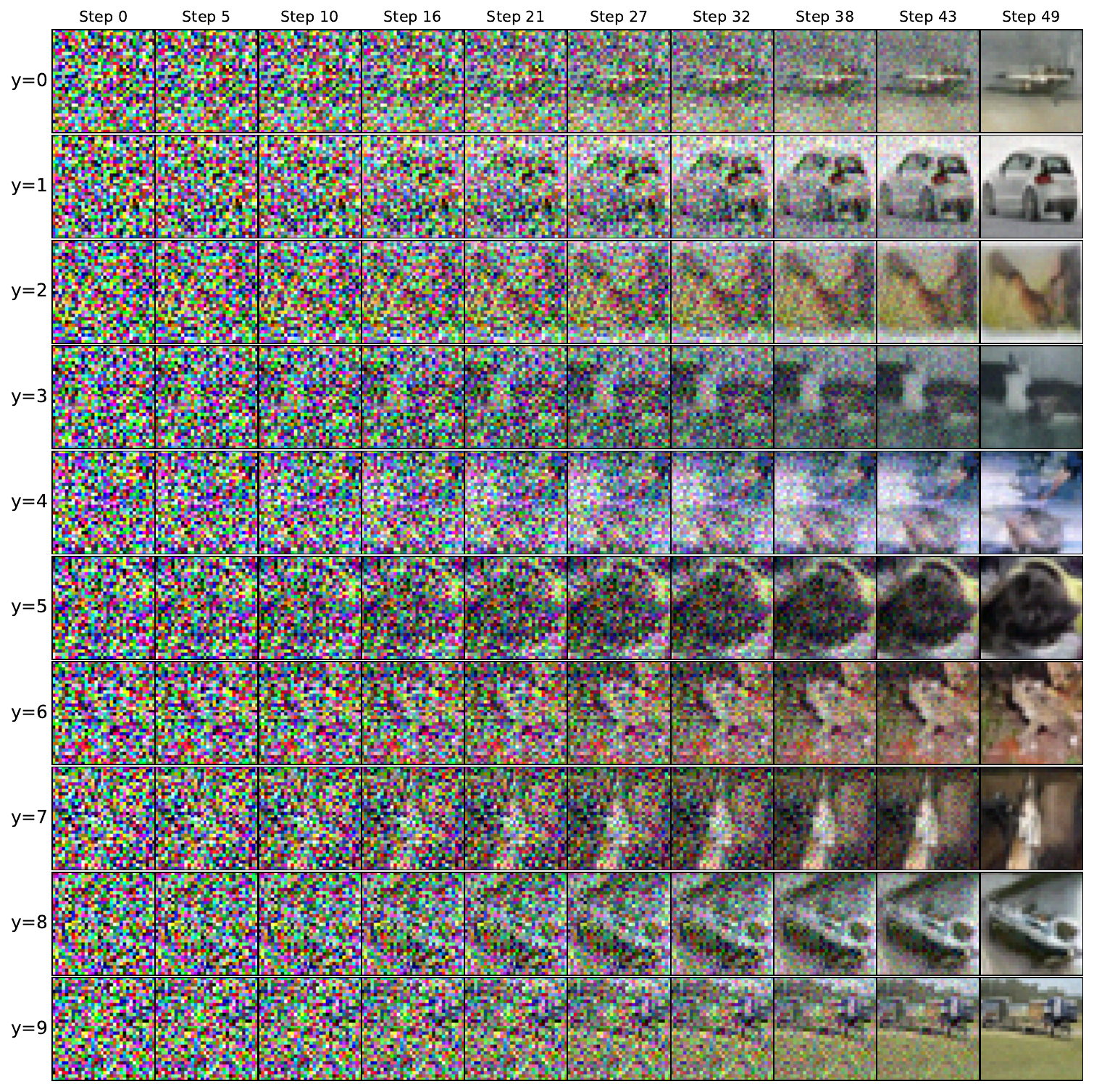}
    \small (a) Standard CFM with Gaussian prior ($w=1.0$)
  \end{minipage}
  \hfill
  \begin{minipage}{0.48\textwidth}
    \centering
    \includegraphics[width=\linewidth]{cifar10_trajectory_sequence_CFG1.pdf}
    \small (b) P-Guide with class-conditioned prior ($w=1.0$)
  \end{minipage}
    \caption{\textbf{Trajectory comparison under different initial latent distributions on CIFAR-10.} 
    Each row visualizes the evolution of a sample from its initial state ($t=0$) to the final generated image ($t=1$). 
    Standard CFM starts from a fixed Gaussian prior that is independent of class labels, resulting in less structured early-stage trajectories. 
    In contrast, P-Guide initializes from a class-conditioned latent distribution, where different categories induce distinct starting points. }
    \label{fig:cfm_vs_pguide_trajs}
\end{figure*}

\begin{figure*}[h]
  \centering
  \scalebox{0.94}{
  \begin{minipage}{0.48\textwidth}
    \centering
    \includegraphics[width=\linewidth]{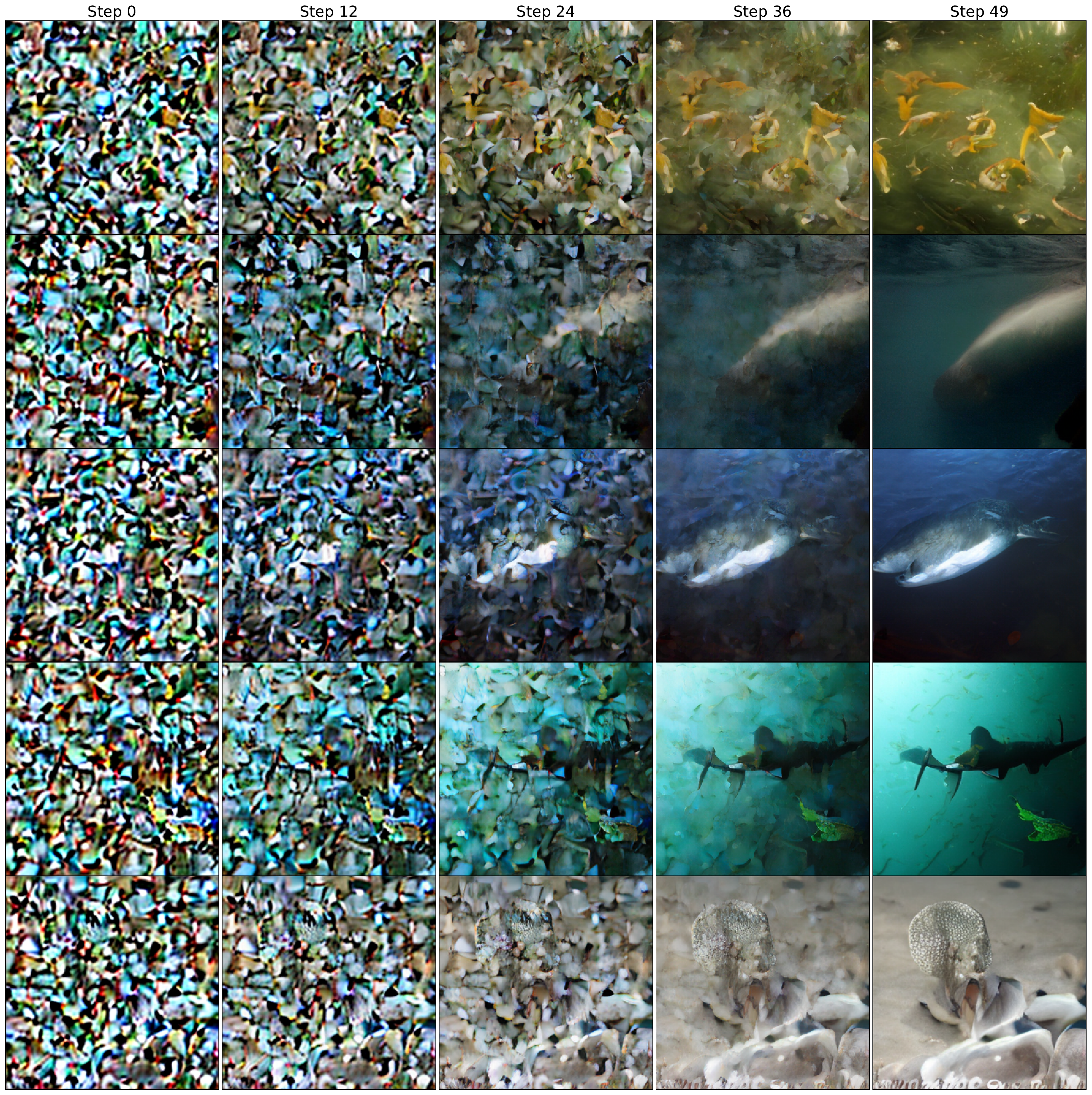}
    \small (a) Standard CFM with Gaussian prior ($w=1.0$)
  \end{minipage}
  \hfill
  \begin{minipage}{0.48\textwidth}
    \centering
    \includegraphics[width=\linewidth]{trajectory_image1000_sequence_CFG1.0.pdf}
    \small (b) P-Guide with class-conditioned prior ($w=1.0$)
  \end{minipage}}
    \caption{\textbf{Trajectory comparison under different initial latent distributions on ImageNet-1k.} 
    Each row shows the evolution of a sample from its initial state ($t=0$) to the final generated image ($t=1$). 
    As in CIFAR-10, standard CFM starts from a class-agnostic Gaussian prior, leading to less structured trajectories at early stages. 
    In contrast, P-Guide initializes from a class-conditioned latent distribution, where different categories correspond to distinct starting points. }
    \label{fig:imagenet_cfm_pguide_trajs}
\end{figure*}

\section{Anonymous Code Release and Reproducibility}
\label{sec:appendix_code}

To facilitate reproducibility and enable further research, we provide an anonymous code release of P-Guide at:

\begin{center}
\url{https://github.com/XinPeng76/P-Guide.git}
\end{center}

\paragraph{Code scope.}
The current anonymous release includes the core implementation of the proposed P-Guide method, in particular the prior-space steering module and the training pipeline used for ImageNet-1k experiments (see \texttt{train\_imagenet256\_DIT\_PGuide.py}). This implementation is sufficient to reproduce the training dynamics and validate the effectiveness of the proposed single-pass CFG inference paradigm.

\section{Statement on the Use of Large Language Models}
\label{app:llm_statement}

During the preparation of this manuscript, large language models (LLMs) were used in a limited manner
solely for language editing purposes, such as improving clarity, grammar, and academic style.
All aspects of the research conception, methodological development, experimental design,
analysis of results, and the scientific conclusions presented in this paper
were carried out independently by the authors.


\end{document}